\documentclass{article}
\PassOptionsToPackage{square,sort, comma, numbers}{natbib}
\usepackage[preprint]{neurips_2020}

\usepackage{wrapfig}
\usepackage{amssymb, amsmath, amsthm}
\usepackage{dsfont, mathrsfs, color, bm, bbm, enumitem, url}
\usepackage{subfigure, graphicx, transparent, booktabs} 
\usepackage{multirow, titlecaps, accents, makebox, epstopdf, mathtools}
\usepackage{caption}

\usepackage{algorithm, algorithmic}
\renewcommand{\algorithmicrequire}{\textbf{Input:}} 
\renewcommand{\algorithmicensure}{\textbf{Output:}}

\theoremstyle{definition}
\newtheorem{theorem}{Theorem}[section]
\newtheorem{definition}[theorem]{Definition}

\newtheorem{proposition}[theorem]{Proposition}
\newtheorem{corollary}[theorem]{Corollary}
\newtheorem{lemma}[theorem]{Lemma}

\newtheorem{example}[theorem]{Example}

\usepackage[utf8]{inputenc}    
\usepackage[T1]{fontenc}       
\usepackage{url}               
\usepackage{booktabs}          
\usepackage{amsfonts}          
\usepackage{nicefrac}          
\usepackage{microtype}         

\newcommand{\be}{\begin{equation}}
\newcommand{\ee}{\end{equation}}
\newcommand{\bea}{\begin{equation*}\begin{aligned}}
\newcommand{\eea}{\end{aligned}\end{equation*}}
\newcommand{\ds}{\displaystyle}

\newcommand{\R}{\mathbb{R}}

\newcommand{\Min}{\min\limits_}
\newcommand{\Sup}{\sup\limits_}

\newcommand{\wh}{\widehat}
\newcommand{\mc}{\mathcal}
\newcommand{\mbb}{\mathbb}

\newcommand{\PP}{\mbb P}
\newcommand{\Pnom}{\wh{\mbb P}}
\newcommand{\QQ}{\mbb Q}
\newcommand{\DD}{\mbb D}
\newcommand{\dd}{\mathrm{d}}

\DeclareMathOperator{\st}{s.t.}

\newcommand{\Let}{\triangleq}
\newcommand{\opt}{^\star}

\newcommand{\Wass}{\mathds{W}}

\newcommand{\EE}{\mathds{E}}

\newcommand{\modif}[1]{#1}

\author{
	Viet Anh Nguyen \qquad \qquad  Fan Zhang \qquad \qquad Jos\'{e} Blanchet\\
	Stanford University, United States \\
	\texttt{ \{viet-anh.nguyen, fzh, jose.blanchet\}@stanford.edu } 
	\AND
	Erick Delage \\
	HEC Montr\'{e}al, Canada \\
	\texttt{erick.delage@hec.ca} 
	\AND
	Yinyu Ye \\
	Stanford University, United States\\
	\texttt{yinyu-ye@stanford.edu} 
}
\newcommand{\quoteIt}[1]{``#1''}
\newcommand{\myRe}{\mathbb{R}}

\title{Distributionally Robust Local Non-parametric Conditional Estimation}
\begin{document}
	\maketitle
	\begin{abstract}
		Conditional estimation given specific covariate values (i.e., local conditional estimation or functional estimation) is ubiquitously useful with applications in engineering, social and natural sciences. Existing data-driven non-parametric estimators mostly focus on structured homogeneous data (e.g., weakly independent and stationary data), thus they are sensitive to adversarial noise and may perform poorly under a low sample size. To alleviate these issues, we propose a new distributionally robust estimator that generates non-parametric local estimates by minimizing the worst-case conditional expected loss over all adversarial distributions in a Wasserstein ambiguity set. We show that despite being generally intractable, the local estimator can be efficiently found via convex optimization under broadly applicable settings, and it is robust to the corruption and heterogeneity of the data. Experiments with synthetic and MNIST datasets show the competitive performance of this new class of estimators.
	\end{abstract}

	\section{Introduction}

	We consider the estimation of conditional statistics of a response variable, $Y\in\myRe^m$, given the value of a predictor or covariate $X\in\myRe^n$. The single most important instance of these types of problems involves estimating the conditional mean, or also known as the regression function. 	Under finite variance assumptions, the conditional mean $\EE_\PP[Y | X = x_0]$ is technically defined as $\psi\opt(x_0)$ for some measurable function $\psi\opt$ that solves the minimum mean square error problem
		\be \notag
		\Min{\psi} ~\EE_\PP[ \| Y - \psi(X) \|_2^2 ],
		\ee
		where the minimization is taken over the space of all measurable functions from $\R^n$ to $\R^m$. While the optimal solution $\psi\opt$ is unique up to sets of $\PP$-measure zero, unfortunately, solving for $\psi\opt$ is challenging because it is an infinite-dimensional optimization problem. The regression function $\psi\opt$ can be efficiently found only under specific settings, for example, if one assumes that $(X, Y)$ follows a jointly Gaussian distribution. However, these specific situations are overly restrictive in practice.
		
		In order to bypass the infinite-dimensional challenge involved in directly computing $\psi\opt$, we may instead consider a family of optimization problems that are parametrized by $x_0$. More specifically, in the presence of a regular conditional distribution, the conditional mean $\EE_{\PP}[Y | X = x_0]$ can be estimated pointwise by $\wh \beta$ defined as
		\[
	     \wh \beta \in \arg \Min{\beta} \EE_{\PP} [ \| Y - \beta \|_2^2 | X = x_0 ]
		\]
		for any covariate value $x_0$ of interest. This presents the challenge of effectively accessing the conditional distribution, which is particularly difficult if the event $X=x_0$ has $\PP$-probability zero.
		
		Using an analogous argument, if we are interested in the conditional $(\tau \times 100\%)$-quantile of $Y$ given $X$, then this conditional statistics can be estimated pointwise at any location $x_0$ of interest by
		\[
		    \wh \beta \in \arg \Min{\beta } \EE_{\PP}[ \max\{-\tau(Y - \beta), (1-\tau)(Y - \beta)\} | X = x_0].
		\]
The previous examples illustrate that the estimation of a wide range of conditional statistics can be recast into solving a family of finite-dimensional optimization problems parametrically in $x_0$
	\be \label{eq:original}
	\Min{\beta}~\EE_{\PP} [ \ell(Y, \beta) | X = x_0]
    \ee
	with an appropriately chosen statistical loss function $\ell$. 
	
	Problem \eqref{eq:original} poses several challenges, some of which were alluded to earlier. First, it requires the integration with respect to a difficult to compute conditional probability distribution. Second, the probability measure $\PP$ is generally unknown, hence we lack a fundamental input to solve \eqref{eq:original}. Finally, in a data-driven setting, there may be few, or even no, observations with value covariate $X = x_0$.
	
	
	To alleviate these difficulties, our formulation, as we shall explain, involves two features. First, we consider a relaxation of problem~\eqref{eq:original} in which the event $X=x_0$ is replaced by a neighborhood $\mc N_{\gamma}(x_0)$ of a suitable radius $\gamma \ge 0$ around $x_0$. Second, we introduce a data-driven distributionally robust optimization (DRO) formulation (e.g. \cite{ref:delage2010distributionally, ref:blanchet2019quantifying, ref:kuhn2019wasserstein}) in order to mitigate the problem that $\PP$ is unknown. In turn, the DRO formulation involves a novel class of conditional ambiguity set which copes with the underlying {\it conditional distribution} being unknown.  
		
	In particular, we propose the following \textit{distributionally robust local conditional estimation problem}
	\be \label{eq:local_DRO}
	    \Min{\beta } \Sup{\QQ \in \mbb B_\rho^\infty, \QQ (X \in \mc N_{\gamma}(x_0)) > 0 } \EE_{\QQ} \big[ \ell(Y, \beta) | X \in \mc N_{\gamma}(x_0) \big],
	\ee
	where the maximization is taken over all probability measures $\QQ$ that are within $\rho$ distance in the $\infty$-Wasserstein sense of a benchmark nominal model, which often corresponds to the empirical distribution of available data. The probability measures $\QQ$ are constrained so that $\QQ(X \in \mc N_{\gamma}(x_0))>0$ to eliminate the complication of conditioning on a set of measure zero.

	\textbf{Contributions.} Resting on formulation \eqref{eq:local_DRO}, our main contributions are summarized as follows.  
	\begin{enumerate}[leftmargin = 5mm]
	    \item We introduce a novel paradigm of non-parametric local \textit{conditional} estimation based on distributionally robust optimization. In contrast to classical non-parametric conditional estimators, our new class of estimators are endowed by design with robustness features. They are structurally built to mitigate the impact of model contamination and therefore they may be reasonably applied to heterogeneous data (e.g., non i.i.d.~input).
	    \item We demonstrate that when the ambiguity set is a type-$\infty$ Wasserstein ball around the empirical measure, the proposed min-max estimation problem can be efficiently solved in many applicable settings, including notably the local conditional mean and quantile estimation.
	    \item We show that this class of type-$\infty$ Wasserstein local conditional estimators can be considered as a systematic robustification of the $k$-nearest neighbor estimator. We also provide further insights on the statistical properties of our approach and empirical evidence, with both a synthetic and real data sets, that our approach can provide more accurate estimations in practically relevant settings.
	\end{enumerate}


\textbf{Related work.} One can argue that every single prediction task in machine learning ultimately relates to conditional estimation. So, attempting to provide a full literature survey on non-parametric conditional estimation is an impossible task. Since our contribution is primarily on introducing a novel conceptual paradigm powered by DRO, we focus on discussing well-understood estimators that encompass most of the conceptual ideas used to mitigate the challenges exposed earlier. 

The challenges of conditioning on zero probability events and the fact that $x_0$ may not be a part of the sample are addressed based on the idea of averaging around a neighborhood of the point of interest and smoothing. This gives rise to estimators such as $k$-NN (see, for example, \cite{ref:devroye1978uniform}), and kernel density estimators, including, for instance the Nadaraya-Watson estimator (\cite{ref:nadaraya1964estimating, ref:watson1964smooth}) and the Epanechnikov estimator \cite{ref:epanechnikov1969non}, among others. Additional averaging methods include, for example, random forests~\cite{ref:breiman2001random} and Classification and Regression Trees (CARTs, \cite{ref:breiman1984classification}), see also \cite{ref:hastie2009elements} for other techniques.  

These averaging and smoothing ideas are well understood, leading to the optimal selection (in a suitable sense) of the kernel along with the associated tuning parameters such as the bandwidth size. These choices are then used to deal with the ignorance of the true data generating distribution by assuming a certain degree of homogeneity in the data, such as stationarity and weak dependence, in order to guarantee consistency and recovery of the underlying generating model. However, none of these estimators are directly designed to cope with the problem of general (potentially adversarial) data contamination. 

The later issue revolving around the evaluation of an unknown conditional probability model is connected with robustness, another classical topic in statistics \cite{ref:huber1992robust}. Much of the classical literature on robustness focuses on the impact of outliers. The work of \cite{ref:zhao2016robust} studies robust-against-outliers kernel regression which enjoys asymptotic consistency and normality under i.i.d.~assumptions in a setting where the data contamination becomes negligible. In contrast to this type of contamination, our estimators are designed to be min-max optimal in the DRO sense by supplying the best response against a large (non-parametric) class of adversarial contamination.   

\modif{Our results can also be seen as connected to adversarial training, which has received a significant amount of attention in recent years \cite{ref:goodfellow2015texplaining,ref:kurakin2016adversarial,ref:madry2018towards,ref:tramer2018ensemble, ref:sinha2017certifying,ref:raghunathan2018certified}. Existing robustification of the nearest neighbors and of the nonparametric classifiers in general can be streamlined into two main strategies: i) global approaches that modify the whole training dataset, e.g., adversarial pruning \cite{ref:wang2018analyzing, ref:yang2020robustness, ref:bhattacharjee2020when}, and ii) local approaches that study well-crafted attack and seek appropriate defense for specific classifiers such as {$1$-NN} \cite{ref:khoury2018geometry, ref:wang2019evaluating,ref:li2019advknn}. Following this line of ideas, one can interpret our approach as a novel method to train conditional estimators against adversarial attacks. The difference, in the $k$-NN estimation setting for example, is that our attacks are optimal in a distributional sense. Our proposed estimator is thus provably the best for a uniform class of distributional attacks. Compared to the current literature, we believe that our approach is also more general in two significant ways: first, we start from a generic min-max estimation problem, and our ideas and methodology are easily applicable to other non-parametric settings, and second, we allow for perturbations on $Y$ to hedge against label contamination. 
}

DRO-based estimators have generated a great deal of interest because they possess various desirable properties in connection to various forms of regularization (e.g., variance~\cite{ref:namkoong2017variance}; norm~\cite{ref:shafieezadeh2019regularization}; shrinkage~\cite{ref:nguyen2018distributionally}). The tools that we employ are related to those currently being investigated. Our formulation considers adversarial perturbations based on the Wasserstein distance~\cite{ref:esfahani2018data, ref:blanchet2019quantifying, ref:gao2016distributionally, ref:kuhn2019wasserstein}. In particular, the type-$\infty$ Wasserstein distance \cite{ref:givens1984class} is recently applied in DRO formulations \cite{ref:bertsimas2018data, ref:bertsimas2019computational, ref:xie2019tractable}. In particular, the work of \cite{ref:bertsimas2019dynamic} considers adversarial conditional estimation, taking as input various classical estimators (e.g., $k$-NNs, kernel methods, etc.) and proposes a robustification approach considering only perturbation in the response variable. Our method whereas allows perturbations both to the covariate and response variables, which is technically more subtle because of the local conditioning problem. 
Within the $k$-NN DRO conditional robustification, our numerical experiments in Section~\ref{sect:numerical} show substantial benefits of our local conditioning approach, especially in dealing with non-homogeneity and sharp variations in the underlying density.

	\textbf{Notations.} For any integer $M \in \mbb N_+$, we denote by $[M]$ the set $\{1, \ldots, M\}$. For any set $\mc S$, $\mc M(\mc S)$ is the space of all probability measures supported on $\mc S$.
	
	\section{Local Conditional Estimate using Type-$\infty$ Wasserstein Ambiguity Set}
	\label{sect:infty-set}
	
	We start by delineating the building blocks of our distributionally robust estimation problem~\eqref{eq:local_DRO}. The nominal measure is set to the empirical distribution of the available data, $\Pnom = N^{-1} \sum_{i \in [N]} \delta_{(\wh x_i, \wh y_i)}$, where $\delta_{(\wh x, \wh y)}$ represents the Dirac distribution at $(\wh x, \wh y)$. The ambiguity set $\mbb B_\rho^\infty$ is a Wasserstein ball around $\Pnom$ that contains the true distribution $\PP$ with high confidence. 
	
	\begin{definition}[Wasserstein distance]
		Let $\DD$ be a metric on $\Xi$. The type-$p$ $(1 \leq p < +\infty)$ Wasserstein distance between $\QQ_1$ and $\QQ_2$ is defined as
		\[
		\Wass_{p}(\QQ_1, \QQ_2) \Let \inf \left\{ \big(\EE_\pi [\DD(\xi_1, \xi_2)^p] \big)^{\frac{1}{p}}:
		\pi \in \Pi(\QQ_1, \QQ_2)
		\right\},
		\]
		where $\Pi(\QQ_1 , \QQ_2)$ is the set of all probability measures on $\Xi \times \Xi$ with marginals $\QQ_1$ and $\QQ_2$, respectively. 
		The type-$\infty$ Wasserstein distance is defined as the limit of $\Wass_p$ as $p$ tends to $\infty$ and amounts to
		\[
		\Wass_{\infty}(\QQ_1, \QQ_2) \Let \inf \left\{ \mathrm{ess} \Sup{\pi} \big\{ \DD(\xi_1, \xi_2) : (\xi_1, \xi_2)  \in \Xi \times \Xi \big\} :
		\pi \in \Pi(\QQ_1, \QQ_2)
		\right\}.
		\]
	\end{definition}
	
	 We assume that $(X,Y)$ admits values in $\mc X \times \mc Y \subseteq \R^n \times \R^m$, and the distance $\DD$ on $\mc X \times \mc Y$ is 
	 \[
	    \DD\big( (x, y), (x', y') \big) = \DD_{\mc X}(x, x') + \DD_{\mc Y}(y, y') \qquad \forall (x, y), (x', y') \in \mc X \times \mc Y,
	 \]
	 where $\DD_{\mc X}$ and $\DD_{\mc Y}$ are continuous metric on $\mc X$ and $\mc Y$, respectively. The joint ambiguity set $\mbb B^\infty_\rho$ is now formally defined as a type-$\infty$ Wasserstein ball in the space of joint probability measures
	\[
	\mbb B^\infty_\rho \Let \left\{ \QQ \in \mc M(\mc X \times \mc Y): \Wass_\infty(\QQ, \Pnom) \leq \rho  \right\}.
	\]
	We assume further that the compact neighborhood $\mc N_\gamma(x_0)$ around $x_0$ is prescribed using the distance $\DD_{\mc X}$ as $\mc N_\gamma(x_0)\Let\{ x \in \mc X: \DD_{\mc X}(x,x_0)\le\gamma\}$, and the loss function $\ell$ is jointly continuous in $y$ and $\beta$.

    To solve the estimation problem~\eqref{eq:local_DRO}, we study the worst-case conditional expected loss function
	\[
	f(\beta) \Let\sup_{\QQ \in \mbb B_\rho, \QQ (X \in \mc N_{\gamma}(x_0)) > 0 } \EE_{\QQ} \big[ \ell(Y, \beta) | X \in \mc N_{\gamma}(x_0) \big],
	\]
	which corresponds to the inner maximization problem of~\eqref{eq:local_DRO}.
	To ensure that the value $f(\beta)$ is well-defined, we first investigate the conditions under which the above supremum problem has a non-empty feasible set. Towards this end, for any set $\mc N_\gamma(x_0) \subset \mc X$, define the quantities $\kappa_{i, \gamma}$ as
		\be \label{eq:kappa-def}
		0 \le \kappa_{i, \gamma} \Let \min_{x \in \mc N_\gamma(x_0)}\DD_{\mc X}(x, \wh x_i) + \inf_{y \in \mc Y}~\DD_{\mc Y}(y, \wh y_i)  \qquad \forall i \in [N].
		\ee
	The value $\kappa_{i, \gamma}$ signifies the unit cost of moving a point mass from an observation $(\wh x_i, \wh y_i)$ to the fiber set $ \mc N_\gamma(x_0) \times \mc Y$. We also define $\wh x_i^p$ as the projection of $\wh x_i$ onto the neighborhood $\mc N_\gamma(x_0)$, which coincides with the optimal solution in the variable $x$ of the minimization problem in~\eqref{eq:kappa-def}. The next proposition asserts that $f(\beta)$ is well-defined if the radius $\rho$ is sufficiently large.

	\begin{proposition}[Minimum radius]
		\label{prop:infty-vanilla-set}
		For any $x_0 \in \mc X$ and $\gamma \in \R_+$, there exists a distribution $\QQ \in \mbb B_\rho$ that satisfies $\QQ(X \in \mc N_\gamma(x_0)) > 0$ if and only if $\rho \ge \min_{i \in [N]} \kappa_{i, \gamma}$.
	\end{proposition}

	\begin{wrapfigure}{R}{0.42\textwidth}
		\centering
		\includegraphics[width=0.4\textwidth]{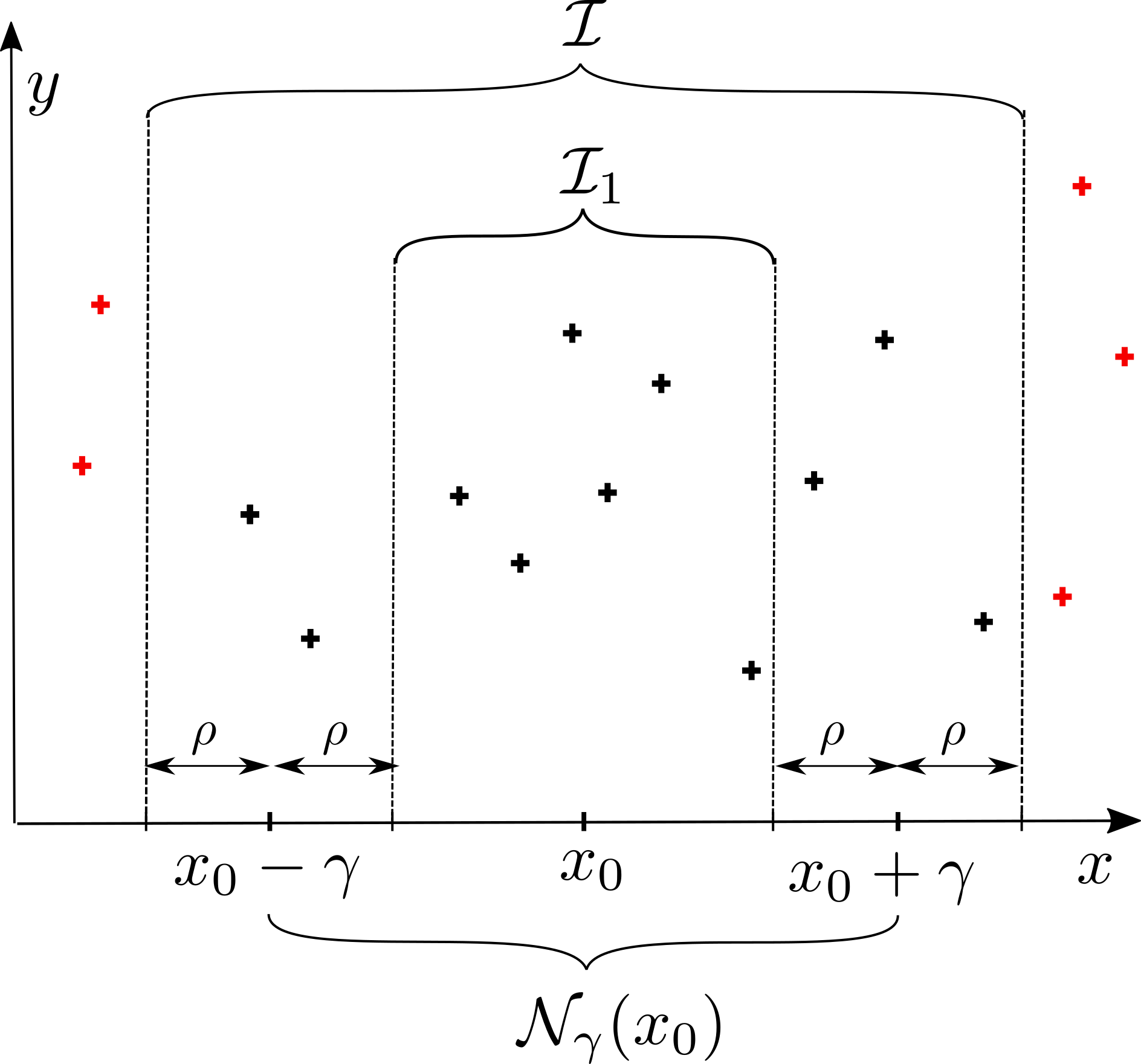}
		\caption{Illustration around the neighborhood of $x_0$ with $\rho < \gamma$. Black crosses are samples in the set $\mc I$.}
		\label{fig:I-set}
	\end{wrapfigure}
	We now proceed to the reformulation of $f(\beta)$. Let $\mc I$ be the index set defined as
    \begin{subequations}
	\be  \label{eq:I-def1}
	\mc I \Let \left\{ i \in [N] : \DD_{\mc X}(x_0, \wh x_i) \le \rho + \gamma \right\},
	\ee
	and $\mc I$ is decomposed further into two disjoint subsets
	\be \label{eq:I-def2}
	    \mc I_1 = \left\{ i \in \mc I: \DD_{\mc X}(x_0, \wh x_i) + \rho \le \gamma \right\} ~~ \text{and} ~~ \mc I_2 = \mc I \backslash \mc I_1.
	\ee
	\end{subequations}
	
	Intuitively speaking, $\mc I$ contains the indices of data points whose covariate $\wh x_i$ is sufficiently close to $x_0$ measured by $\DD_{\mc X}$, and are thus relevant to the local estimation problem. The index set $\mc I_1$ indicates the data points that lie strictly inside the neighborhood, while the set $\mc I_2$ contains those points that are on the boundary ring of width $\rho$ around the neighborhood $\mc N_{\gamma}(x_0)$. The value $f(\beta)$ can be efficiently computed in a quasi-closed form thanks to the following result.

	\begin{theorem}[Worst-case conditional expected loss computation]
		\label{thm:infty-refor}
		For any $\gamma \in \R_+$, suppose that $\rho \ge \min_{i \in [N]} \kappa_{i, \gamma}$. For any $\beta \in \mc Y$, let $v_i\opt(\beta)$ be defined as
		\begin{equation}
			\label{eq:vi-def}
			v_i \opt(\beta) \Let \Sup{y_i} \left\{ \ell(y_i, \beta) : y_i \in \mc Y,~\DD_{\mc Y}(y_i, \wh y_i) \leq \rho - \DD_{\mc X}(\wh x_i^p, \wh x_i) \right\}	\quad \forall i \in \mc I.
		\end{equation}		    
		The worst-case conditional expected loss is equal to 
		$
		f(\beta) =
		\big(\sum_{i \in \mc I} \alpha_i\big)^{-1} \sum_{i \in \mc I}  \alpha_i v_i\opt(\beta),
		$ 
		where $\alpha$ admits the value 
        \be \notag 
        \forall i \in \mc I: \quad \alpha_i = \begin{cases}
            1 & \text{if } i \in \mc I_1 \text{ or } (\mc I_1 = \emptyset \text{ and } v_i\opt(\beta) = \max_{j \in \mc I_2} v_j\opt(\beta)), \\
            1 & \text{if } \ds v_i\opt(\beta) > \frac{\sum_{i \in \mc I_1} v_i\opt(\beta) + \sum_{j \in \mc I_2: v_j\opt(\beta) > v_i\opt(\beta)} v_j\opt(\beta)}{|\mc I_1| + |\{j \in \mc I_2: v_j\opt(\beta) > v_i\opt(\beta) \}|},\\
            0 & \text{otherwise.}
        \end{cases} 
    \ee
	\end{theorem}

	If we possess an oracle that evaluates~\eqref{eq:vi-def} at a complexity $\mc O$, then by Theorem~\ref{thm:infty-refor}, quantifying $f(\beta)$ is reduced to calculating $|\mc I|$ values of $v_i\opt(\beta)$ and then sorting these values in order to determine the value of $\alpha$. Thus, computing $f(\beta)$ takes an amount of time of order $O\big(| \mc I | (\log | \mc I| + \mc O)\big)$. Moreover,  $f(\beta)$ depends solely on the observations in the locality of $x_0$ whose indices belong to the index set $\mc I$, the cardinality of which can be substantially smaller than the total number of training samples $N$.
	
	If $\ell$ is a convex function in $\beta$, then a standard result from convex analysis implies that $f$, being a pointwise supremum of convex functions, is also convex.
	If $\mc Y$, and hence $\beta$, is unidimensional, a golden section search algorithm can be utilized to identify the local conditional estimate $\beta^*$ that solves~\eqref{eq:local_DRO} in an amount of time of order $O\big(\log(1/\epsilon) |\mc I| (\log(|\mc I|)+\mc O)\big)$, where $\epsilon> 0$ is an arbitrary accuracy level. Fortunately, in the case of conditional mean and quantile estimation, we also have access to the closed form expressions of $v_i\opt(\beta)$ as long as
	$\DD_{\mc Y}$ is an absolute distance.
	
	\begin{corollary}[Value of $v_i\opt(\beta)$] \label{corol:vi-opt}
	Suppose that $\mc Y= [a, b] \subseteq [-\infty, +\infty]$ and $\DD_{\mc Y}(y_i, \wh y_i)=|y_i- \wh y_i|$.
	\begin{enumerate}[label=(\roman*), leftmargin=6mm]
		\item Conditional mean estimation: if $\ell(y, \beta) = (y - \beta)^2$, then $\forall i \in \mc I$
		\[
		    v_i\opt(\beta) = \max\big\{ (\max\{ \wh y_i + \rho - \DD_{\mc X}(\wh x_i^p, \wh x_i), a\} - \beta)^2, (\min\{ \wh y_i + \rho - \DD_{\mc X}(\wh x_i^p, \wh x_i), b\} - \beta)^2
		    \big\}.
		\]
		\item Conditional quantile estimation: if $\ell(y, \beta)=\max\{-\tau(y - \beta), (1-\tau)(y - \beta)\}$, then $\forall i \in \mc I$
		\[
		    v_i\opt(\beta) \!=\!\max\big\{-\tau(\max\{ \wh y_i + \rho - \DD_{\mc X}(\wh x_i^p, \wh x_i), a\} - \beta), (1-\tau)(\min\{ \wh y_i + \rho - \DD_{\mc X}(\wh x_i^p, \wh x_i), b\} - \beta)\big\}.
		\]
	\end{enumerate}
	\end{corollary}

If $\mc Y$ is multidimensional, the structure of $\ell(y,\beta)$ and $\DD_{\mc Y}$ might be exploited to identify tractable optimization reformulations. The next result focuses on the local conditional mean estimation.
\begin{proposition}[Multivariate conditional mean estimation]  \label{prop:cond-exp}
		Let $\mc Y = \R^m$ and $\ell(y, \beta) = \| y - \beta \|_2^2$.
		\begin{enumerate}[label=(\roman*), leftmargin=6mm]
			\item \label{item:cond-exp-2}  Suppose that $\DD_{\mc Y}$ is a 2-norm on $\mc Y$, that is, $\DD_{\mc Y}(y, \wh y) = \| y - \wh y\|_2$. The distributionally robust local conditional estimation problem~\eqref{eq:local_DRO} is equivalent to the second-order cone program
			\[
             \begin{array}{cll}
			        \min & \lambda \\
			        \st & \beta\in \R^m,\;\lambda \in \R,\;u_i \in \R\;\forall i \in \mc I_1,\;u_i \in \R_+\;\forall i \in \mc I_2,\;t_i \in \R_+\;\forall i \in \mc I \\
			        & \sum_{i \in \mc I} u_i \le 0, \quad t_i \geq \|\wh y_i-\beta\|_2 \quad \forall i \in \mc I\\
			        & \|[t_i+\rho-\DD_{\mc X}(\wh x_i^p, \wh x_i)\;;\;(1/2)(1-\lambda-u_i)]\|_2\le (1/2)(1+\lambda + u_i) \quad \forall i \in \mc I.
			    \end{array}
			\]
			\item \label{item:cond-exp-infty}
			Suppose that $\DD_{\mc Y}$ is a $\infty$-norm on $\mc Y$, that is, $\DD_{\mc Y}(y, \wh y) = \| y - \wh y\|_\infty$. The distributionally robust local conditional estimation problem~\eqref{eq:local_DRO} is equivalent to the second-order cone program
			\[
		    \begin{array}{cll}
		        \min & \lambda \\
		        \st & \beta\in \R^m,\;\lambda \in \R,\; T \in \R_+^{|\mc I| \times m},\;u_i \in \R \;\forall i \in \mc I_1,\;u_i \in \R_+\;\forall i \in \mc I_2 \\
		        & \sum_{i \in \mc I} u_i \le 0, \;;\; \|[T_{i1}\;;\; T_{i2}\;;\;\cdots\;;\;T_{im}\;;\;\frac{1}{2}(1-\lambda-u_i)]\|_2\leq \frac{1}{2}(1+\lambda + u_i) \;\; \forall i \in \mc I \\
		       & \hspace{-2mm}\left. \begin{array}{l}
		        T_{ij} \leq \wh y_{ij} - \beta_j - \rho + \DD_{\mc X}(\wh x_i^p, \wh x_i) \leq T_{ij}  \\
		        T_{ij}\leq \wh y_{ij} - \beta_j + \rho - \DD_{\mc X}(\wh x_i^p, \wh x_i) \le T_{ij}
		        \end{array} \right\} \forall (i, j) \in \mc I \times [m],
		    \end{array}
		    \]
			where $\wh y_{ij}$ and $\beta_j$ are the $j$-th component of $\wh y_i$ and $\beta$, respectively. 
		\end{enumerate}
	\end{proposition}

Both optimization problems presented in Proposition~\ref{prop:cond-exp} can be solved in large scale
by commercial optimization solvers such as MOSEK \cite{mosek}. For other multivariate conditional estimation problems, there is also a possibility of employing subgradient methods by leveraging on the next proposition.


\begin{proposition}[Subgradient of $f$] \label{prop:gradient}
     Suppose that $\DD_{\mc Y}$ is coercive and $\ell(y,\cdot)$ is convex. Under the conditions of Theorem~\ref{thm:infty-refor}, for any $\beta \in \R^m$, a subgradient of the function $f$ at $\beta$ is given by
	$
	\partial f(\beta) = (\sum_{i \in \mc I} \alpha_i)^{-1} \sum_{i \in \mc I}  \alpha_i \partial_\beta \ell(y_i\opt, \beta)$, where the value of $\alpha$ is as defined in Theorem~\ref{thm:infty-refor} and  $y_i\opt$ satisfies $y_i\opt \in \{y_i\in \mc Y:\DD_{\mc Y}(y_i, \wh y_i) \le \rho - \DD_{\mc X}(\wh x_i^p, \wh x_i),\;\ell(y_i\opt, \beta) = v_i\opt(\beta)\}$ for all $i \in \mc I$.
\end{proposition}

    \modif{Just as an adversarial example provides a description on how to optimally perturb a data point from the adversary's viewpoint~\cite{ref:khoury2018geometry, ref:wang2019evaluating}, the worst-case distribution provides full information on how to adversarially perturb the empirical distribution $\Pnom$. For our distributionally robust estimator, the worst-case distribution can be obtained from the result of Theorem~\ref{thm:infty-refor}.
    \begin{lemma}[Worst-case distribution]
        Fix an estimate $\beta \in \mc Y$. Suppose that $\rho \ge \min_{i \in [N]} \kappa_{i, \gamma}$ and let $v\opt(\beta)$ and $\alpha$ be determined as in Theorem~\ref{thm:infty-refor}. Moreover, let $y_i\opt$ satisfy $y_i\opt \in \{y_i\in \mc Y:\DD_{\mc Y}(y_i, \wh y_i) \le \rho - \DD_{\mc X}(\wh x_i^p, \wh x_i),\;\ell(y_i\opt, \beta) = v_i\opt(\beta)\}$ for all $i \in \mc I$. Then the distribution 
        \[
            \QQ\opt = \frac{1}{N} \left( \sum_{i\in \mc I: \alpha_i = 1} \delta_{(\wh x_i^p, y_i\opt)} + \sum_{i \in \mc I: \alpha_i = 0} \delta_{(\wh x_i, \wh y_i)} + \sum_{i \in [N] \backslash \mc I} \delta_{(\wh x_i, \wh y_i)} \right)
        \]
        satisfies $f(\beta) = \EE_{\QQ\opt} \big[ \ell(Y, \beta) | X \in \mc N_{\gamma}(x_0) \big]$.
    \end{lemma}
    
    The values of $\alpha$ calculated in Theorem~\ref{thm:infty-refor} are of indicative nature: $\alpha_i = 1$ if it is optimal to perturb the sample point $i$ to compute the worst-case conditional expected loss. The construction of the worst-case distribution is hence intuitive: it involves computing and sorting the values $v\opt_i(\beta)$, and then performing a greedy assignment in order to maximize the objective value.
    }
	\section{Probabilistic Theoretical Properties}
	\label{sect:guarantee}
	 
    We now study the some statistical properties of our proposed estimator. Under some regularity conditions, the type-$\infty$ Wasserstein ball can be viewed as a confidence set that contains the true distribution $\PP$ with high probability, provided that the radius $\rho$ is chosen judiciously. The value $f(\beta\opt)$ thus constitutes a generalization bound on the out-of-sample performance of the optimal conditional estimate $\beta\opt$. This idea can be formalized as follows.

    \begin{proposition}[Finite sample guarantee] \label{prop:finite}
        Suppose that $\mc X \times \mc Y$ is bounded, open, connected with a Lipschitz boundary. Suppose that the true probability measure $\PP$ of $(X, Y)$ admits a density function $\nu$ satisfying $\bar{\nu}^{-1} \le \nu(x, y) \le \bar \nu$ for some constant $\bar \nu \ge 1$. For any $\gamma > 0$, if 
        \[
            \rho \ge \begin{cases}
                C N^{-\frac{1}{2}} \log(N)^{\frac{3}{4}} & \text{when } n + m = 2, \\
                C N^{-\frac{1}{n+m}} \log(N)^{\frac{1}{n+m}} & \text{otherwise},
            \end{cases}
        \]
        where $C$ is a constant dependent on $\mc X\times\mc Y$ and $\bar \nu$, then for a probability of at least $1-O(N^{-c})$, where ${c\!>\!1}$ is a constant dependent on $C$, we have
       $\EE_{\PP}[\ell(Y, \beta\opt) | X \in \mc N_{\gamma}(x_0)] \le f(\beta\opt)$,
        where $\beta\opt$ is the optimal conditional estimate that solves problem~\eqref{eq:local_DRO}.
    \end{proposition}

    We now switch gear to study the properties of our estimator in the asymptotic regime, in particular, we focus on the consistency of our estimator. The interplay between the neighborhood radius $\gamma$ and the ambiguity size $\rho$ often produces tangling effects on the asymptotic convergence of the estimate. We thus showcase two exemplary setups with either $\gamma$ or $\rho$ is zero, which interestingly produce two opposite outcomes on the consistency of the estimator. This underlines the intricacy of the problem.
    \begin{example}[Non-consistency when $\gamma = 0$] \label{example:non-consistency}
        Suppose that $\gamma = 0$, $\rho \in \R_{++}$ be a fixed constant, $\mc Y = \R$, $\ell(y, \beta) = (y - \beta)^2$,  and $\DD_{\mc Y}$ is the absolute distance. Let $ \beta_{N}\opt$ be the optimal estimate that solves~\eqref{eq:local_DRO} dependent on $\{(\wh x_i, \wh y_i)\}_{i=1, \ldots, N}$. If under the true distribution $\PP$, $X$ is independent of $Y$,
        $\PP(\DD_{\mc X}(X,x_0)\leq\rho) > 0$, $\PP(Y\!\geq\!0)\!=\!1$ and $\PP(Y\!\geq\!y)\!>\!0~ \forall y\!>\!0$, then with probability 1, we have
        $     
        \wh \beta_{N} \rightarrow +\infty$ while $\EE_{\PP}[Y|X\!=\!x_0]\!<\!\infty$.
    \end{example}

    \begin{example}[Consistency when $\rho = 0$] \label{example:k-nn}
        Suppose that $\rho = 0$, $\mc Y = \R$, $\ell(y, \beta) = (y - \beta)^2$, $\DD_{\mc X}$ and $\DD_{\mc Y}$ are the Euclidean distance, $k_N$ is a sequence of integer. Let $\gamma$ be the $k_N$-th smallest value of $\DD_{\mc X}(x_0, \wh x_i)$, then $\beta_{N}\opt$ that solves~\eqref{eq:local_DRO} recovers the $k_N$-nearest neighbor regression estimator.
        If $k_N$ satisfies $\lim_{N\rightarrow \infty} k_N = \infty$ and $\lim_{N\rightarrow \infty} k_N/N = 0$, and $\EE_{\PP}[Y|X=x]$ is a continuous function of $x$, then 
        $\lim_{N\rightarrow\infty}\beta_{N}\opt = \EE_{\PP}[Y|X=x_0]$ by \cite[Corollary~3]{ref:stone1977consistent}.
    \end{example}

    Example~\ref{example:k-nn} suggests that if the radius $\gamma$ of the neighborhood is chosen adaptively based on the available training data, then our proposed estimator coincides with the $k$-nearest neighbor estimator, and hence consistency is inherited in a straightforward manner. The robust estimator with an ambiguity size $\rho > 0$ and an adaptive neighborhood radius $\gamma$ can thus be considered as a robustification of the $k$-nearest neighbor, which is obtained in a systematic way using the DRO framework.
    
    It is desirable to provide a descriptive connection between the distributionally robust estimator vis-\`{a}-vis some popular statistical quantities. 
    For the local conditional mean estimation, our estimate $\beta\opt$ coincides with the conditional mean of the distribution with the highest conditional variance. This insight culminates in the next proposition and bolsters the explainability of this class of estimators.
	\begin{proposition}[Conditional mean estimate] \label{prop:max-var}
	    Suppose that $\mc Y = \R$, $\ell(y, \beta) = ( y - \beta )^2$ and $\DD_{\mc Y}(\cdot, \wh y)$ is convex, coercive for any $\wh y$. For any $\rho \ge \min_{i \in [N]} \kappa_{i, \gamma}$, define $\QQ\opt$ as
	    \be \notag
	        \QQ\opt = \arg\max_{\QQ \in \mbb B_\rho^\infty, \QQ(X \in \mc N_\gamma(x_0)) > 0} ~\textrm{Variance}_{\QQ}(Y | X \in \mc N_\gamma(x_0)),
	    \ee
	    then $\beta\opt = \EE_{\QQ\opt}[Y | X \in \mc N_\gamma(x_0)]$ is the optimal estimate that solves  problem~\eqref{eq:local_DRO}.
	\end{proposition}

	\section{Numerical Experiment}
	\label{sect:numerical}
    
    \newcommand{\mewn}{{DRCME} }
    \newcommand{\Bertsimas}{{BertEtAl} }
    In this section we compare the quality of our proposed Distributionally Robust Conditional Mean Estimator (DRCME) to $k$-nearest neighbour ($k$-NN), Nadaraya-Watson (N-W), and Nadaraya-Epanechnikov (N-E) estimators, together with the robust $k$-NN approach in \cite{ref:bertsimas2019dynamic} (BertEtAl) using a synthetic and the MNIST datasets.
	
	\begin{figure}[ht]
    \centering
    \begin{minipage}{0.48\textwidth}
    \centering
     \includegraphics[width=0.90\textwidth]{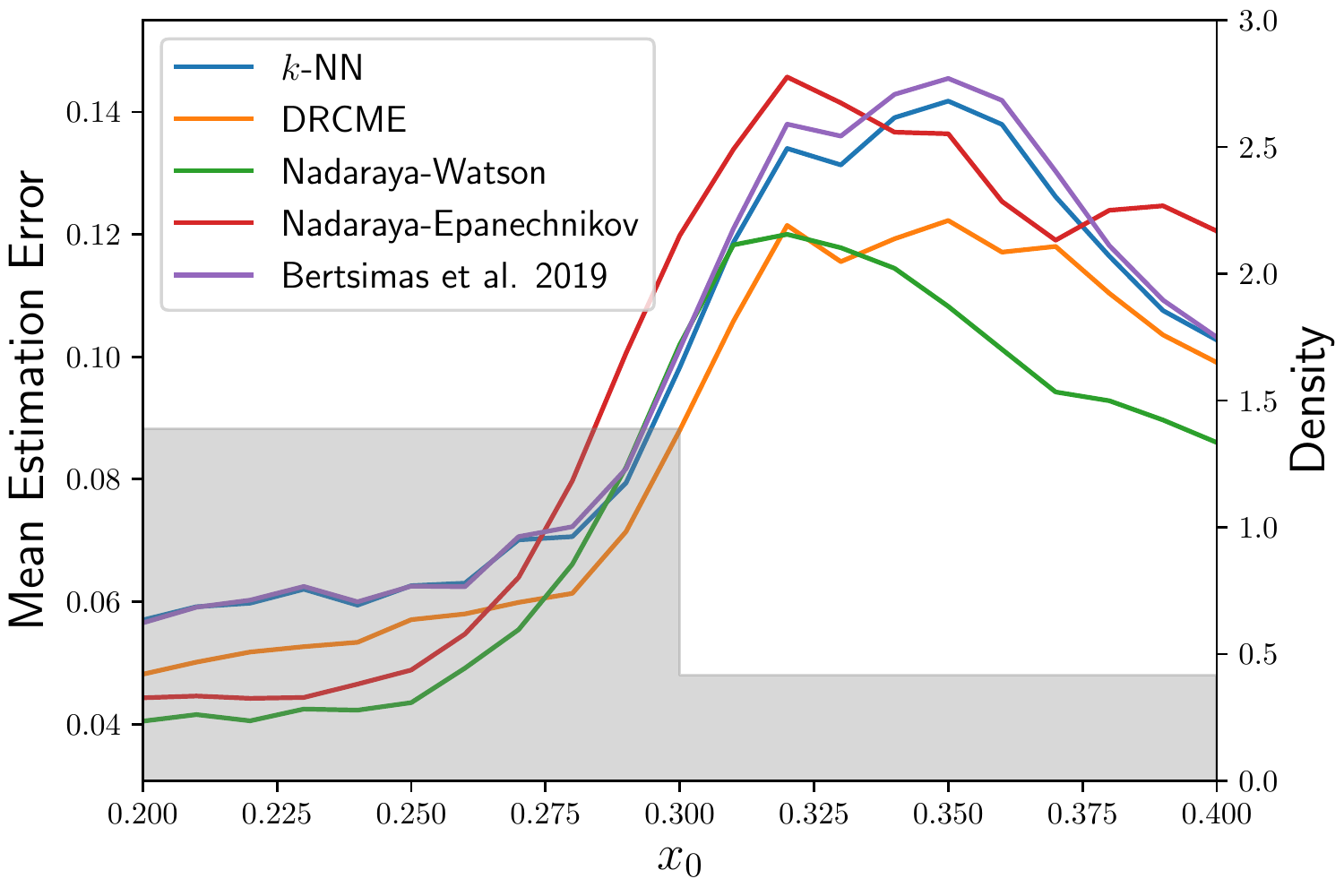} 
    \caption{Comparison of the mean absolute errors of conditional mean estimators for synthetic data. The gray shade shows the density of $X$.}
    \label{fig:local-error}
    \end{minipage}\hfill
    \begin{minipage}{0.48\textwidth}
    \centering
    \includegraphics[width=0.90\textwidth]{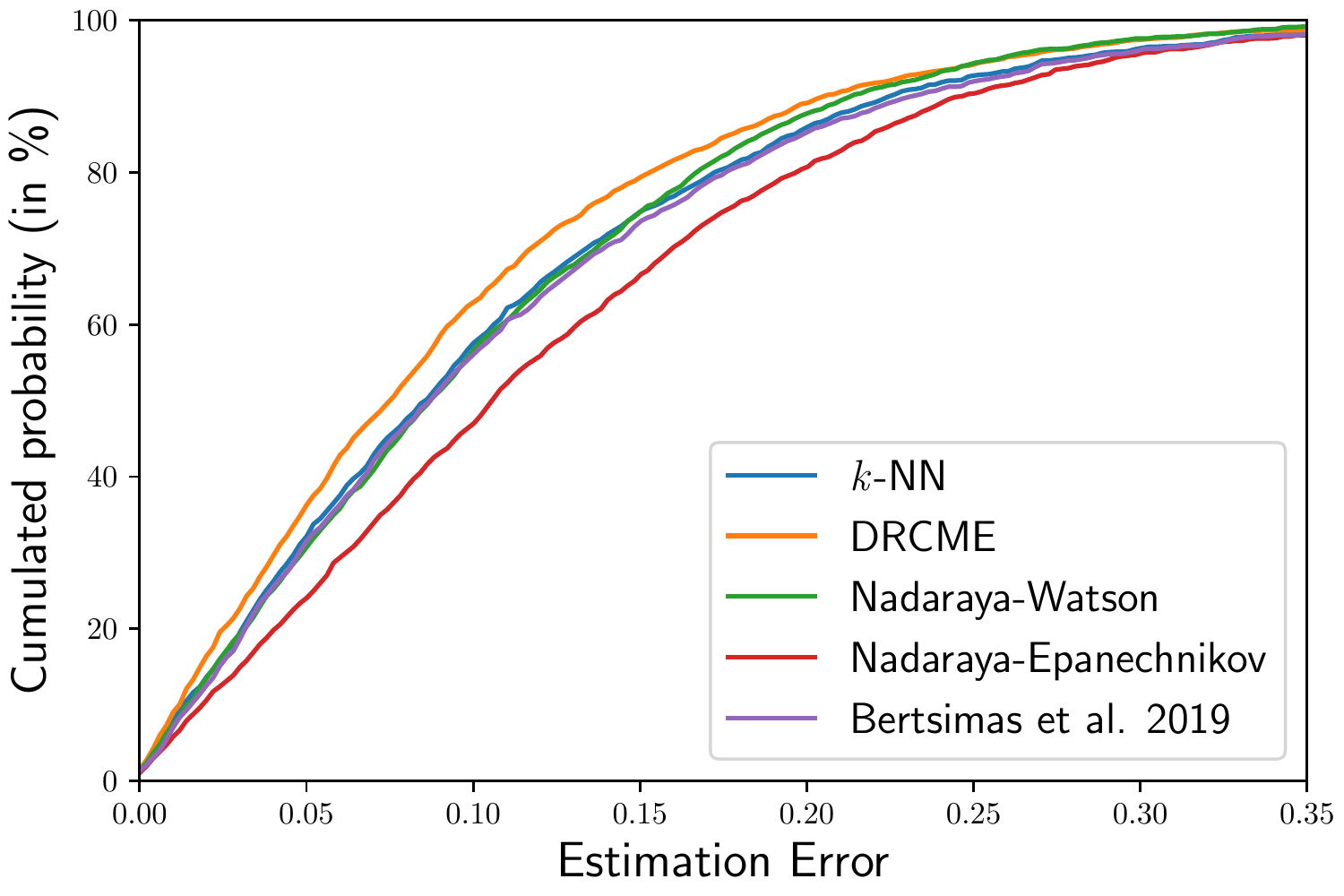} 
    \caption{Comparison of the distributions of absolute estimation errors of conditional mean estimators for synthetic data.}
    \label{fig:local-cdf}
    \end{minipage}
    \vspace{-3mm}
\end{figure} 
	\subsection{Conditional Mean Estimation With Synthetic Data}

	In this section, we conducted $500$ independent experiments where the training set contains $N = 100$ i.i.d.~samples of $(X,Y)$ in each experiment. The marginal distribution of $X$ has piecewise constant density function $p(x)$, which is chosen as $p(x) = 100/72$ if $x\in[0,0.3]\cup[0.7,1]$ and $p(x) = 30/72$ if $x\in(0.3,0.7)$. Given $X$, the distribution of $Y$ is determined by $Y = f(X) + \varepsilon$, where $f = \sin(10\cdot x)$ and $\varepsilon$ is i.i.d.~Gaussian noise independent of $X$ with mean $0$ and variance $0.01$. The conditional mean estimation problem is challenging when $x_0$ is close to the jump points of the density function $p(x)$, that is at $x_0 = 0.3$ or $x_0 = 0.7$, because the data are gathered unequally in the neighborhoods. Thus, to test the robustness of all the estimators, we employ all the five estimators to estimate the conditional mean $\EE_{\PP} [Y| X = x_0]$, for $x_0 = 0.2,0.21,\ldots, 0.4$ around the jump point $x_0 = 0.3$. We select $\DD_{\mc X}(x,x') = |x-x'|$ and $\DD_{\mc Y}(y,y') = |y-y'|$. The hyperparameters of all the estimators, whose range and selection are given in Appendix~\ref{sect:add-result}, are chosen by leave-one-out cross validation.
    
    Figure \ref{fig:local-error} displays the average of the mean estimation errors taken over 500 independent runs for different values $x_0\in[0.2,0.4]$. One can observe from the figure that \mewn uniformly outperforms $k$-NN, \Bertsimas for all $x_0$ of interest. When compared with N-W and N-E, we remark that \mewn is the most accurate estimator around the jump point of $p(x)$. As $x_0$ moves away from the location $0.3$, the performance of \mewn decays and becomes slightly worse than N-W as $x_0$ goes far from the jump point. 
    Figure~\ref{fig:local-cdf} presents the cumulative distribution of the estimation errors when $x_0\in[0.28,0.32]$. The empirical error distribution of \mewn is stochastically smaller than that of other estimators, which reinforces that \mewn outperforms around the jump point in a strong sense. 

	\subsection{Digit Estimation With MNIST Database}

In this section, we compare the quality of the estimators on a digit estimation problem using the MNIST database~\cite{MNIST}. While to this date most studies have focused on out-of-sample classification performances for this dataset,
here we shift our attention to the task of estimation of digits as \textbf{cardinal} quantities and are especially interested in performance at a low-data regime. Treating the labels as cardinal quantities allows us to assess the distinctive features of \mewn in its most simplistic form (i.e. univariate conditional mean estimation of a real random variable). Mean estimation might in fact be more relevant than classification when trying to recognize handwritten measurements where confusing a 0 with a 6 is more damaging than  with a 3.
\begin{table}
\parbox{.48\linewidth}{
\centering
\begin{tabular}{ |l|c|c|c|c| } 
 \hline
 Method & H.P. &$N$=50 & $N$=100 & $N$=500\\
 \hline
 $k$-NN & $k$&3 & 4 & 4\\ 
 \hline
 N-W & $h$ & 0.022 & 0.019 & 0.015 \\ 
 \hline
 N-E & $h$ & 0.087 & 0.078 & 0.068 \\ 
\hline
\Bertsimas & $k$ & 3 & 4 & 5 \\
  & $\rho$ & 0.712 & 1.313 & 1.313\\
 \hline
 & $\gamma$ & $h^\gamma_{1.3}(\cdot)$ & $h^\gamma_{1.3}(\cdot)$ & $h^\gamma_{1.6}(\cdot)$ \\
 \mewn & $\rho$ & 0.13$\gamma$ & 0.13$\gamma$ & 0.06$\gamma$\\
 & $\theta$ & 0.004& 0.002 & 0.001\\
 \hline
\end{tabular}
\caption{Median of hyper-parameters (H.P.) obtained with cross-validation.}\label{table:MNISThyperParams}
}
\hspace{0.7cm}
\parbox{.4\linewidth}{
\centering
\begin{tabular}{ |l|c|c|c| } 
 \hline
 Method & $N$=50 & $N$=100 & $N$=500\\
 \hline
 $k$-NN & $24 \pm 2$ & $35 \pm 2$ & $60 \pm 1$ \\ 
 \hline
 N-W &  $30 \pm 2$ & $38 \pm 2$ & $65 \pm 1$ \\ 
 \hline
 N-E &  $26 \pm 1$ & $32 \pm 1$ & $50 \pm 1$ \\ 
 \hline
 \Bertsimas&  $29 \pm 2$ & $41 \pm 2$ & $67 \pm 1$ \\
 \hline
 \mewn&  $36 \pm 2$ & $46 \pm 2$ & $71 \pm 1$ \\
 \hline
\end{tabular}
\caption{Comparison of expected out-of-sample classification accuracy (in \% with 90\% confidence intervals) from rounded estimates.}\label{table:accuracyMNIST}
}
\vspace{-8mm}
\end{table}

We executed 100 experiments where training and test sets were randomly drawn without replacement from the 60,000 training examples of this dataset. Training set sizes were $N=50$, 100, or 500 while test sets' size remained at 100. Each $(x,y)$ pair is composed of the normalized vector, in $\myRe^{28^2}$ of grayscale intensities normalized so that $\|x\|_1=1$. For simplicity, we let $\DD_{\mc X}(x,\wh x)=\|x-\wh x\|_2$ and $\DD_{\mc Y}(y,\wh y)=\theta|y-\wh y|$.
In each experiment, the hyper-parameters of all four methods were chosen based on a leave-one-out cross validation process. In the case of \mewn\!, we adapt the radius of the neighborhood $\gamma$ and $\rho$ locally at $x_0$ to account for the non-uniform density of $X$.\footnote{
Specifically, we let $\gamma=h^\gamma_i(x_0):=\kappa_{[\lfloor i\rfloor],0}+(i-\lfloor i\rfloor)(\kappa_{[\lceil i\rceil],0}-\kappa_{[\lfloor i\rfloor],0})$, where $[j]$ refers to the $j$-th smallest element while $\lfloor \cdot\rfloor$ and $\lceil \cdot\rceil$ refer to the floor and ceiling operations, i.e., the radius is set to the linear interpolation between the distance of the $\lfloor i\rfloor$-th and $\lfloor i\rfloor+1$-th closest members of the training set to $x_0$. We further let $\rho$ be proportional to $\gamma$. This lets \mewn reduce to $k$-NN when $\gamma=h^\gamma_k(x_0)$, $\rho=0$, and $\theta=1$.} Table \ref{table:MNISThyperParams} presents the median choice of hyper parameters for each estimator.

	\begin{figure}[!ht]
		\centering
		\includegraphics[width=0.9\textwidth]{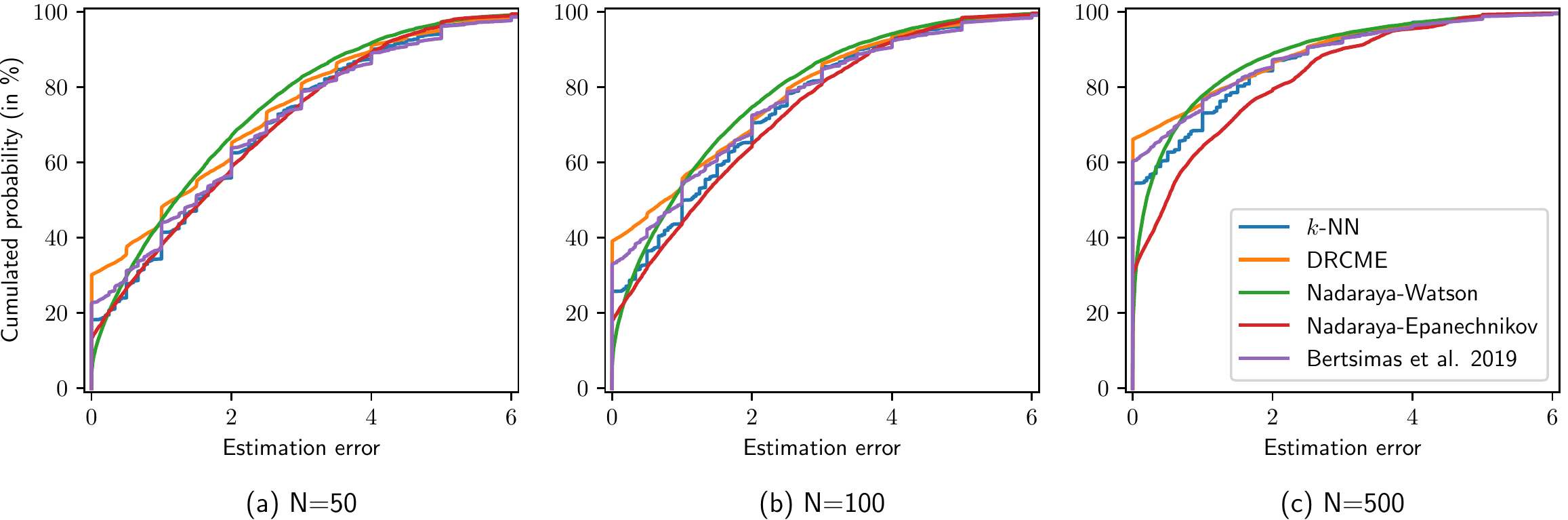}
		\caption{Comparison of the distributions of out-of-sample absolute estimation errors of conditional mean estimators for the MNIST database under different training set sizes.}
		\label{fig:errorDistMNIST}
	\end{figure}

\begin{figure}[!ht]
		\centering
		\includegraphics[width=0.49\textwidth]{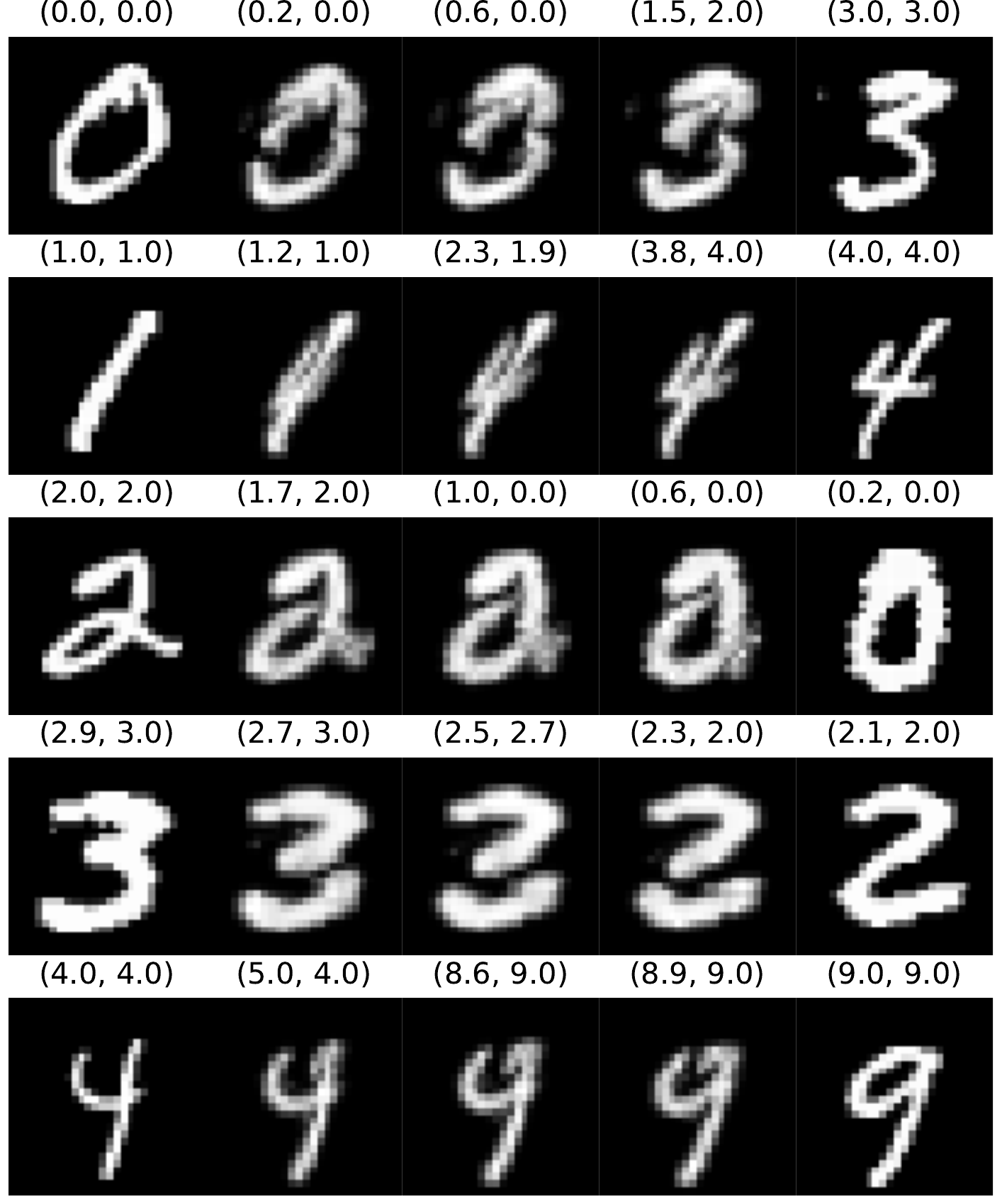}\hspace{0.015\textwidth}\includegraphics[width=0.49\textwidth]{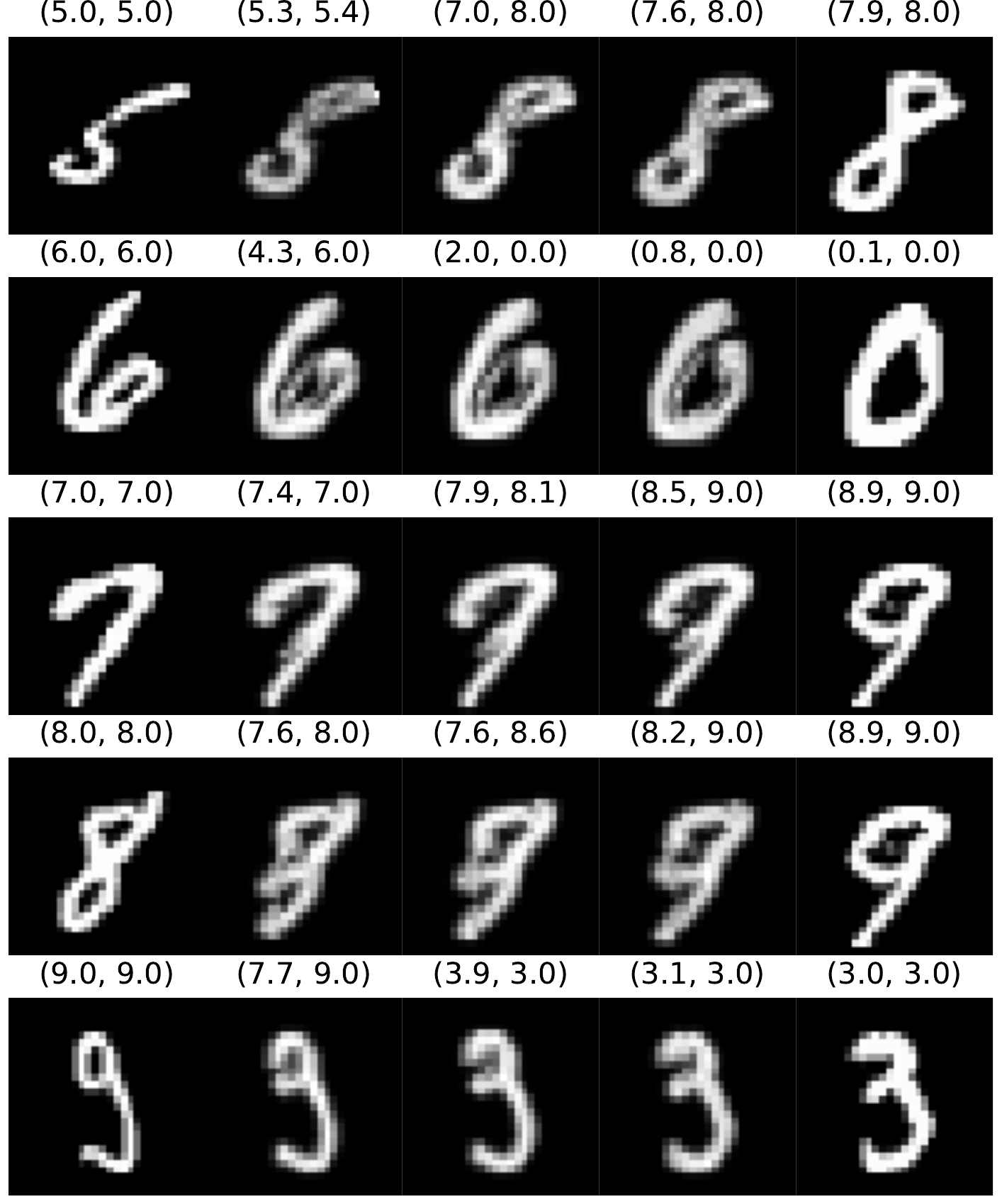}
		\caption{Comparison of estimations from N-W and \mewn on entropic regularized Wasserstein barycenters of pairs of images from the training set. Estimations are presented above each image in the format \quoteIt{(N-W,\;\mewn)}.}
		\label{fig:advMNIST}
		\vspace{-5mm}
	\end{figure}

Figure \ref{fig:errorDistMNIST} presents the out-of-sample estimation error distribution of all four conditional estimators. One can quickly remark that the \mewn outperforms \Bertsimas, $k$-NN, and N-E estimators, especially for low-data regime. In particular, for all three training set sizes, the distribution of error for \mewn stochastically dominates the three other distributions.
In particular, one even notices in (c) that \mewn has the largest chance of reaching an exact estimation: 66\% compared to 60\%, 55\%, 30\%, and 8\% for the other estimators. 
This explains why \mewn is also the most accurate estimator when rounding it to the nearest integer as reported in Table \ref{table:accuracyMNIST}: with a margin greater than 4\% from all estimators across all $N$'s.
It is worth noting that while N-W does not produce high accuracy estimate, it however has less chances of producing estimation with large errors. This is also apparent when comparing the expected type-$p$ deviation of the estimation error, i.e. $(\EE[|y-\hat{y}|^p])^{1/p}$, for each estimator. Specifically, N-W slightly outperforms \mewn for deviation metrics of type $p\geq 1$, e.g. with a root mean square error of 1.32 compared to 1.41 when $N=500$. On the other hand, \mewn significantly outperforms N-W when $p<1$ where high precision estimators are encouraged. We refer the reader to Appendix~\ref{sect:add-result} for further details.

Finally, we report on an experiment that challenges the capacity of both N-W and \mewn estimators to be resilient to adversarial corruption of the test images. This is done by exposing the two estimators to images from the training set ($N=100$) that have been corrupted in a way that makes them resemble the closest differently-labeled image in the set.\footnote{Implementation wise, we exploit the Python Optimal Transport toolbox \cite{flamary2017pot} to compute different entropic regularized Wasserstein barycenters of the two normalized images treated as distributions.} Figure \ref{fig:advMNIST} presents several visual examples of the progressively corrupted images and the resulting N-W and \mewn estimations. Overall, one quickly notices how the estimation produced by \mewn is less sensitive to such attacks, \quoteIt{sticking} to the original label until there is substantial evidence of a new label. More examples are in Appendix~\ref{sect:add-result}.

\paragraph{\bf Acknowledgments.}
Material in this paper is based upon work supported by the Air Force Office of Scientific Research under award number FA9550-20-1-0397. Additional support is gratefully acknowledged from NSF grants 1915967, 1820942, 1838676 and from the China Merchant Bank.

	\appendix
	 
    \section{Additional Experiment Results}
    \label{sect:add-result}

    \subsection{Conditional Mean Estimation With Synthetic Data}
	
	We report in Figure \ref{fig:local-error-full} the plot of mean estimation errors versus $x_0$ for different training set sizes $N = 50, 100, 200$. In Figure \ref{fig:local-cdf-full} we present the plot of the distribution of absolute estimation errors for $x_0\in [0.28,0.32]$. For comparison, we also include the results of training set size $N = 100$ that are already reported in Figure~\ref{fig:local-error} and~\ref{fig:local-cdf}. We remark that the estimation error of all the estimators becomes smaller when training set size is larger, and \mewn has best estimation performance among all the estimators around the jump point $x = 0.3$ for all different training set sizes. 
	
	We report the hyper-parameters selected by cross-validation in Table \ref{table:synthetichyperParams}.
	
	\begin{figure}[!h]
		\centering
         \includegraphics[width=\textwidth]{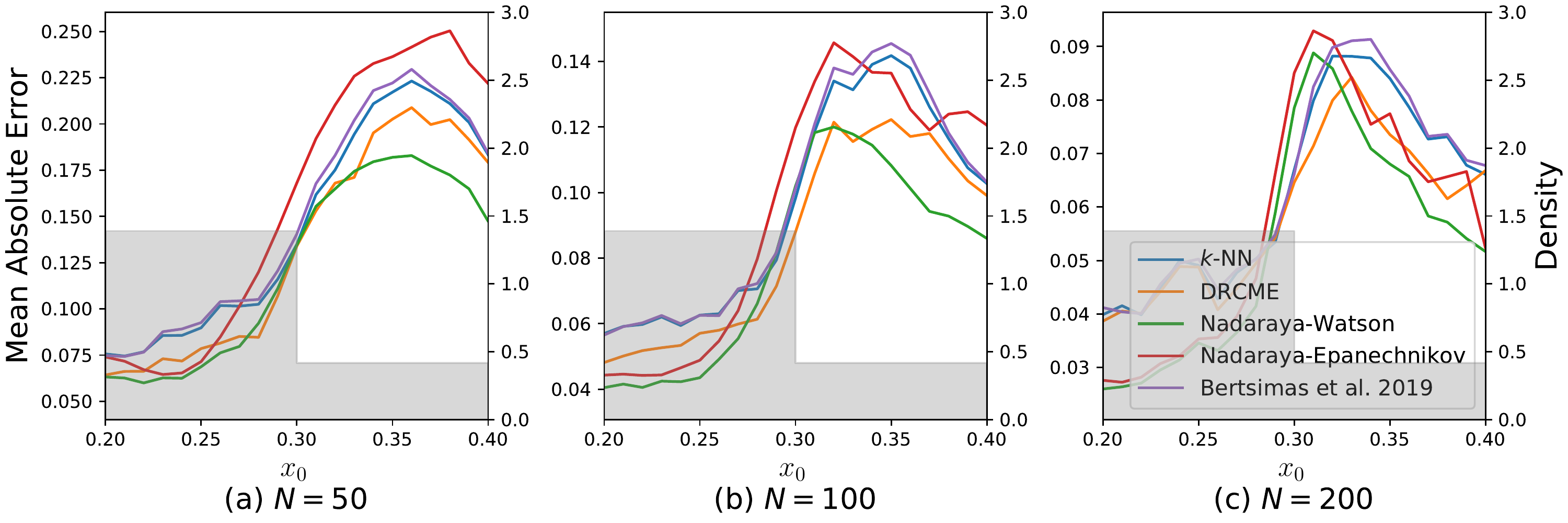}
        \caption{Comparison of the mean absolute errors of conditional mean estimators for synthetic data under different training set sizes. The gray shade shows the density of $X$.}
        \label{fig:local-error-full}
	\end{figure}
	
	\begin{figure}
		\centering
         \includegraphics[width=\textwidth]{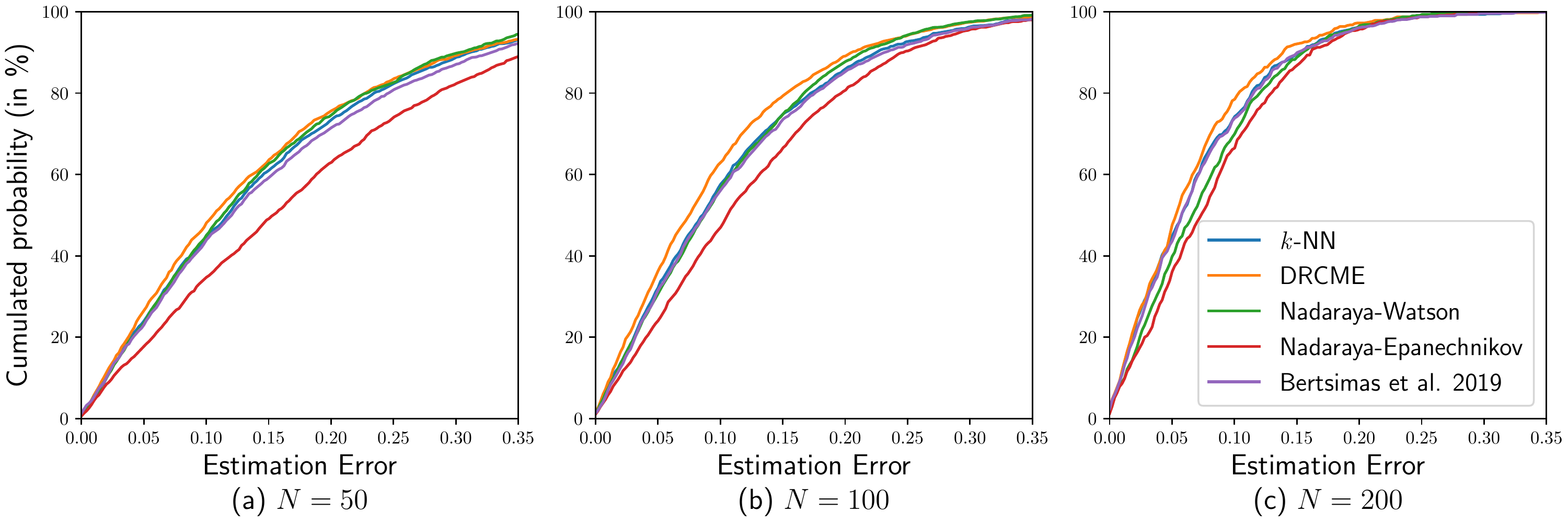}
        \caption{Comparison of the distributions of absolute estimation errors of conditional mean estimators for synthetic data under different training set sizes. }
        \label{fig:local-cdf-full}
	\end{figure}
	
	\begin{table}
    \centering
        \begin{tabular}{ |l|c|c|c|c| } 
         \hline
         Method & H.P. &$N$=50 & $N$=100 & $N$=200\\
         \hline
         $k$-NN & $k$&1 & 3 & 5\\ 
         \hline
         N-W & $h$ & 0.026 & 0.019 & 0.018 \\ 
         \hline
         N-E & $h$ & 0.078 & 0.055 & 0.038 \\ 
        \hline
        \Bertsimas & $k$ & 1 & 3 & 5 \\
          & $\rho$ & 0.063 & 0.016 & 0.000\\
         \hline
         & $\gamma$ & $h^\gamma_{1}(\cdot)$ & $h^\gamma_{2}(\cdot)$ & $h^\gamma_{3}(\cdot)$ \\
         \mewn & $\rho$ & 0.031$\gamma$ & 0.063$\gamma$ & 0.063$\gamma$\\
         \hline
        \end{tabular}
        \caption{Median of hyper-parameters (H.P.) for synthetic data experiment obtained with cross-validation.}\label{table:synthetichyperParams}
    \end{table}
    
    \subsection{Digit Estimation With MNIST Database}

    The distinction between N-W and \mewn is also apparent in Figure \ref{fig:typePdeviationMNIST} which presents the normalized expected type-$p$ deviation of the estimation error for each estimator, i.e. $\sqrt{2/p}(\EE[|y-\hat{y}|^p])^{1/p}$. Specifically, N-W slightly outperforms \mewn for deviation metrics of type $p\geq 1$, e.g. with a root mean square error of 1.34 compared to 1.45 when $N=500$. On the other hand, \mewn significantly outperforms N-W when $p<1$ where high precision estimators are encouraged.

	\begin{figure}[!h]
		\centering
		\includegraphics[width=\textwidth]{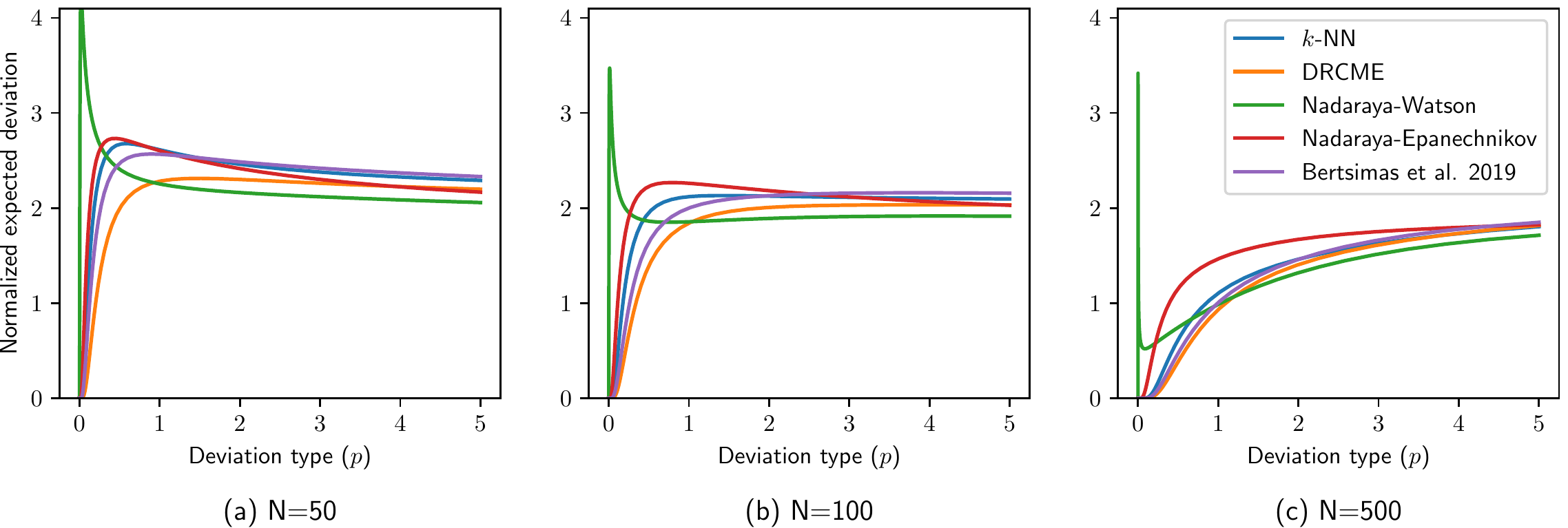}
		\caption{Comparison of normalized expected type-$p$ deviation of the out-of-sample error of four non-parametric conditional mean estimation methods for the MNIST database under different training set sizes. For example, $p=2$ is presented the root-mean square error.}
		\label{fig:typePdeviationMNIST}
	\end{figure}
    
    We also include in Figure \ref{fig:advMNISTExtra1} some additional examples of labels from \mewn and N-W. On the other hand, Figure \ref{fig:advMNISTExtraBert} compares the labels from \mewn and \Bertsimas.

	\begin{figure}[!ht]
		\centering
		\includegraphics[width=0.49\textwidth]{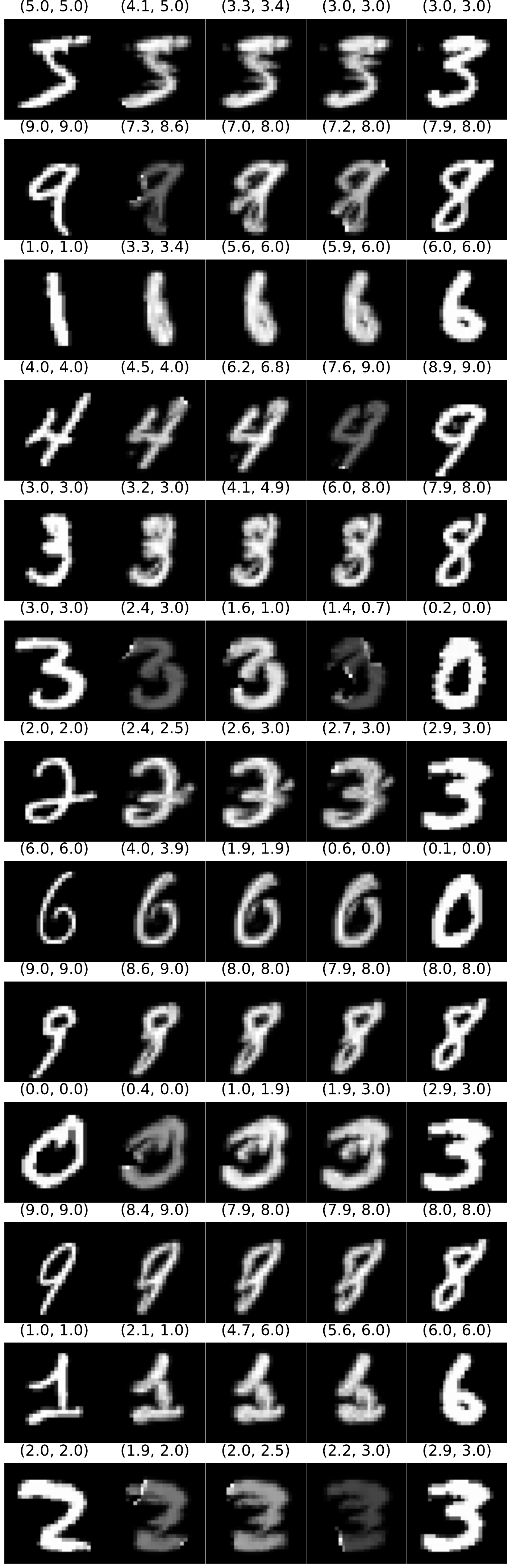}\hspace{0.015\textwidth}\includegraphics[width=0.49\textwidth]{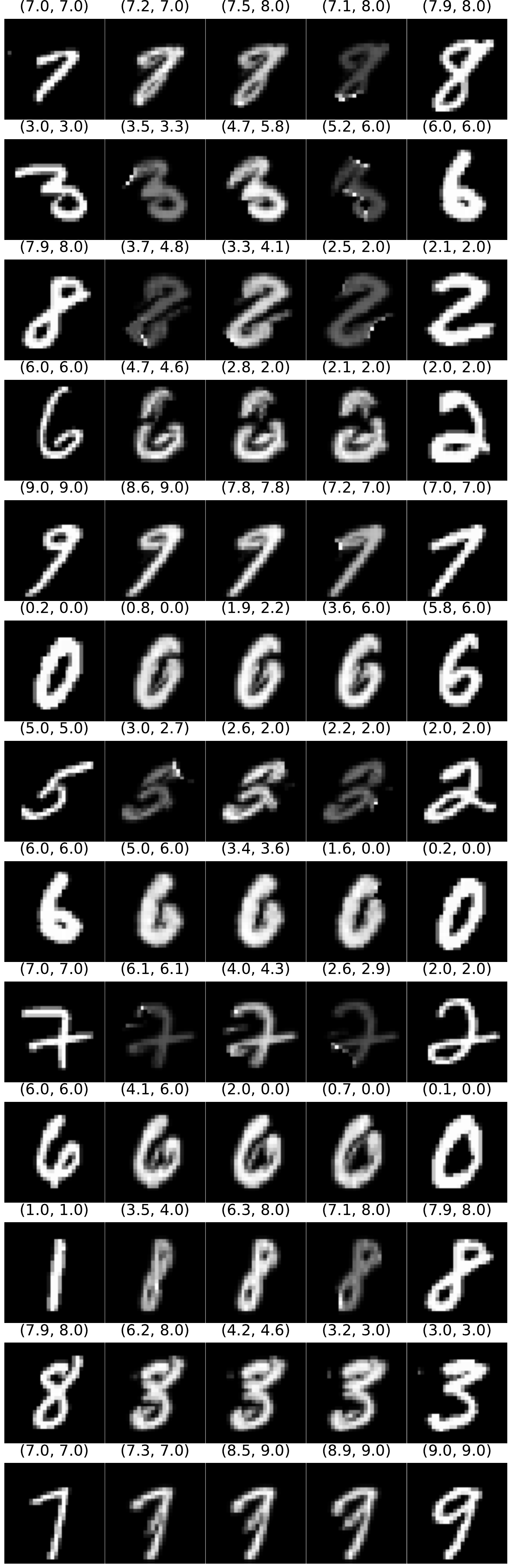}
		\caption{Comparison of estimations from N-W and \mewn on entropic regularized Wasserstein barycenters of pairs of images from the training set. Estimations are presented above each image in the format \quoteIt{(N-W,\;\mewn)}.}
		\label{fig:advMNISTExtra1}
	\end{figure}

	\begin{figure}[!ht]
		\centering
		\includegraphics[width=0.49\textwidth]{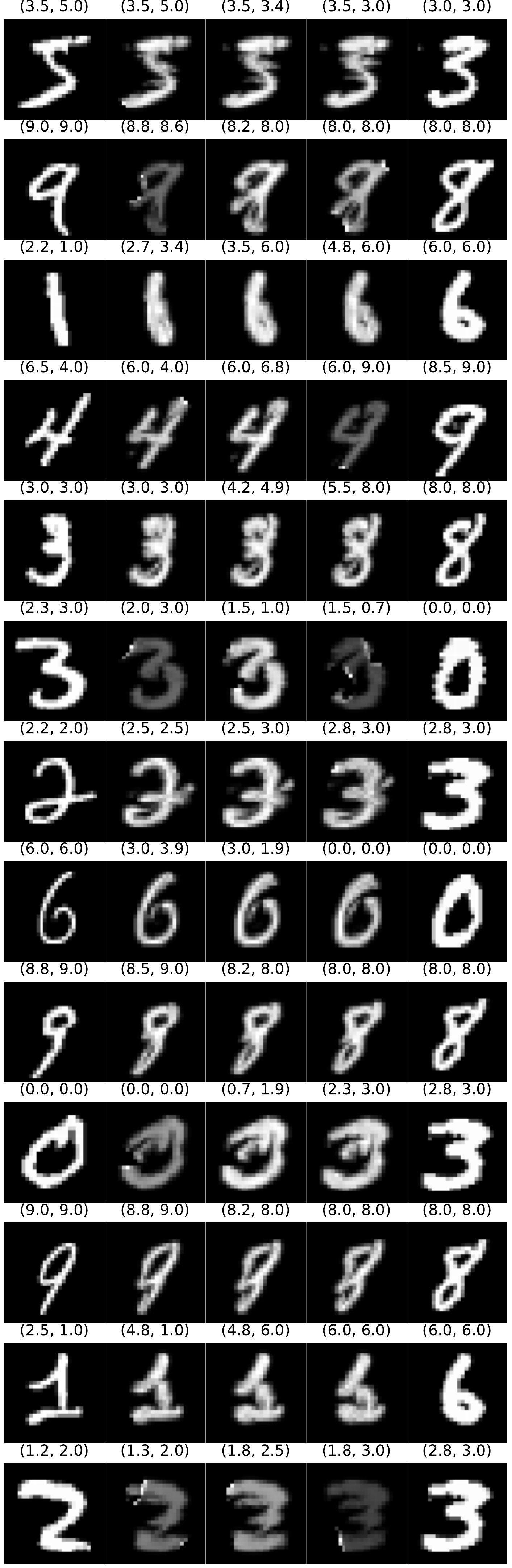}\hspace{0.015\textwidth}\includegraphics[width=0.49\textwidth]{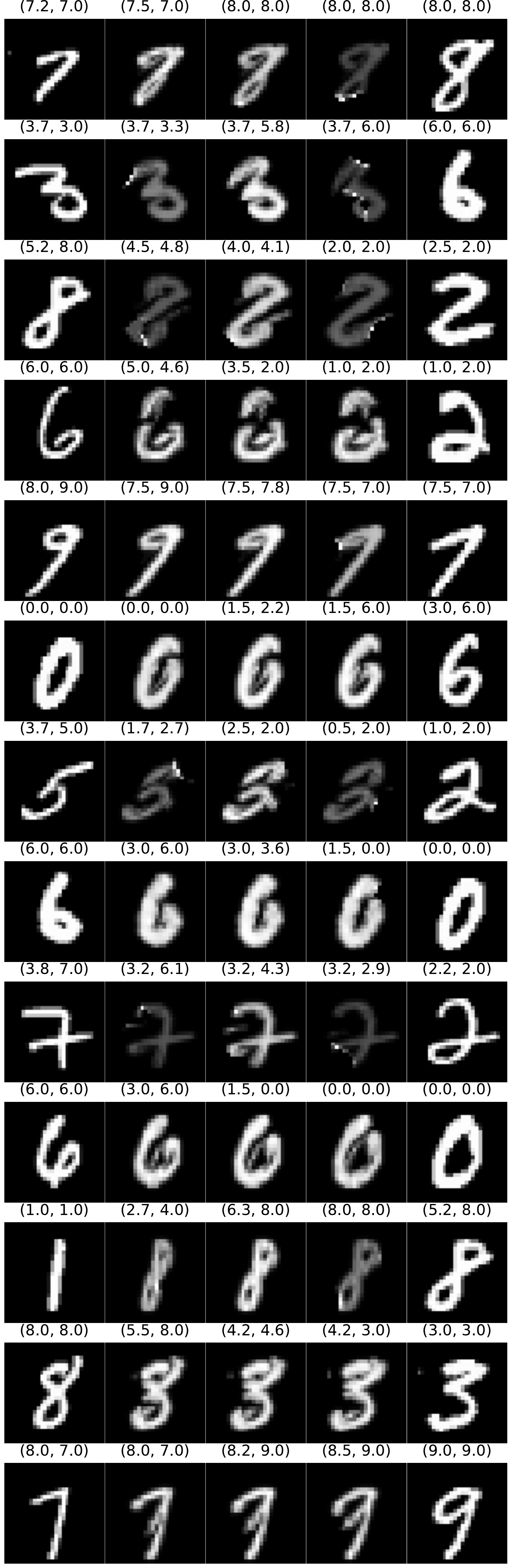}
		\caption{Comparison of estimations from \Bertsimas and \mewn on entropic regularized Wasserstein barycenters of pairs of images from the training set. Estimations are presented above each image in the format \quoteIt{(\Bertsimas\!,\;\mewn\!)}.}
		\label{fig:advMNISTExtraBert}
	\end{figure}	     

	\newpage
	\section{Proofs}
	
	This section contains the proofs of all technical results presented in the main paper.
	
	\subsection{Proofs of Section~\ref{sect:infty-set}}
	
	\begin{proof}[Proof of Proposition~\ref{prop:infty-vanilla-set}]
		Using the definition of the type-$\infty$ Wasserstein distance, we can re-express the ambiguity set $\mbb B^\infty_\rho$ as
		\begin{align*}
		\mbb B^\infty_\rho 
		=& \left\{  \QQ \in \mc M(\mc X \times \mc Y):
		\begin{array}{l}
		\exists \pi \in \Pi(\QQ, \Pnom) \text{ such that }\\
		\mathrm{ess} \Sup{\pi } \{\DD_{\mc X}(x, x') + \DD_{\mc Y}(y, y')\} \leq \rho
		\end{array}
		\right\} \\
		=&\left\{ \QQ \in \mc M(\mc X \times \mc Y):
		\begin{array}{l}
		\exists \pi_i \in \mc M(\mc X \times \mc Y) ~\forall i \in [N] \text{ such that } \QQ = \frac{1}{N} \sum_{i\in[N]} \pi_i\\
		\ds \mathrm{ess} \Sup{\frac{1}{N} \sum_{i \in [N]} \pi_i \otimes \delta_{(\wh x_i, \wh y_i)}} \big\{\DD_{\mc X}(x, x') + \DD_{\mc Y}(y, y')\big\} \leq \rho
		\end{array}
		\right\},
		\end{align*}
		where in the second equality we exploit the fact that $\Pnom$ is an empirical measure and thus any joint probability measure $\pi \in \Pi(\QQ, \Pnom)$ can be written as $\pi = N^{-1} \sum_{i \in [N]} \pi_i \otimes \delta_{(\wh x_i, \wh y_i)}$, where each $\pi_i$ is a probability measure supported on $\mc X \times \mc Y$. The last constraint can now be written as
		\[
		\DD_{\mc X}(x, \wh x_i) + \DD_{\mc Y}(y, \wh y_i) \leq \rho \quad \forall (x, y) \in \mathrm{supp}(\pi_i) \quad \forall i \in [N],
		\]
		where $\mathrm{supp}(\pi_i)$ denotes the support of the probability measure $\pi_i$ \cite[Page~441]{ref:aliprantis06hitchhiker}. We thus have 
		\begin{align*}
		\mbb B^\infty_\rho 
		&= \left\{ \QQ \in \mc M(\mc X \times \mc Y):
		\begin{array}{l}
		\exists \pi_i \in \mc M(\mc X \times \mc Y) ~\forall i \in [N] \text{ such that } \QQ = \frac{1}{N} \sum_{i\in[N]} \pi_i \\ \DD_{\mc X}(x, \wh x_i) + \DD_{\mc Y}(y, \wh y_i) \leq \rho \quad \forall (x, y) \in \mathrm{supp}(\pi_i) \quad \forall i \in [N]
		\end{array}
		\right\}.
		\end{align*}
		Suppose that $\rho < \min_{i \in [N]} \kappa_{i, \gamma}$, then this implies by the last constraint of the feasible set that $\pi_i(\mc N_\gamma(x_0) \times \mc Y) = 0$ for all $i \in [N]$. As a consequence, any $\QQ \in \mbb B_\rho^\infty$ should satisfy
		\[
		    \QQ( X \in \mc N_\gamma(x_0)) = \sum_{i \in [N]} \pi_i(\mc N_\gamma(x_0) \times \mc Y) = 0.
		\]
		Hence $\mbb B_\rho^\infty \cap \{ \QQ \in \mc M(\mc X \times \mc Y): \QQ(X \in \mc N_\gamma(x_0)) > 0\} = \emptyset$. 
		
		Suppose on the contrary that $\rho \ge \min_{i \in [N]} \kappa_{i, \gamma}$. Let $i\opt = \arg\min_{i \in [N]} \kappa_{i, \gamma}$, and consider the following set of probability measures
		\[
		    \forall i \in [N]: \qquad \pi_i = \begin{cases}
		        \delta_{(\wh x_i^p, \wh y_i)} & \text{if } i = i\opt, \\
		        \delta_{(\wh x_i, \wh y_i)} &\text{otherwise,}
		    \end{cases}
		\]
		and set $\QQ = \frac{1}{N} \sum_{i \in [N]} \pi_i$. It is easy to verify that $\QQ \in \mbb B_\rho^\infty$, and that \[\QQ(X \in \mc N_\gamma(x_0)) \ge \frac{1}{N} \pi_{i\opt}(X \in \mc N_\gamma(x_0)) = \frac{1}{N} > 0.\] This observation completes the proof.
	\end{proof}

The proof of Theorem~\ref{thm:infty-refor} relies on the following result.
\begin{lemma}[Optimal solution of a fractional linear program] \label{lemma:fractional}
	    Let $d$ be an strictly positive integer. The linear fractional program
	    \[
	        \min \left\{ \frac{c + \sum_{i=1}^K v_i \alpha_i}{d + \sum_{i=1}^K \alpha_i} : \alpha \in [0, 1]^K  \right\}
	    \]
	    admits the optimal solution
	    \[
	        \forall i \in [K] : \qquad \alpha_i\opt = \begin{cases}
	            1 & \text{if } \ds v_i > \frac{c + \sum_{j: v_j > v_i} v_j}{d + |\{j: v_j > v_i \}|},\\
	            0 & \text{otherwise.}
	        \end{cases}
	    \]
	\end{lemma}
	\begin{proof}[Proof of Lemma~\ref{lemma:fractional}]
	Without loss of generality assume that $v_i$ are ordered decreasingly. Because  the  objective  function  is  pseudolinear,  the  optimal  solution is at some binary vertex~\cite[Lemma~3.3]{ref:kruk1999psedulinear}. Consider the equivalent problem
	\[
	        \max_{k,\alpha} \left\{ \frac{c + \sum_{i=1}^K v_i \alpha_i}{d + \sum_{i=1}^K \alpha_i} : \alpha \in \{0, 1\}^K ,\;\sum_{i=1}^K \alpha_i = k,\; k \in [K] \right\}.
	    \]
	For any value $k \in [K]$, the corresponding optimal value of $\alpha$ dependent on $k$ is
	\[
	    \alpha_i\opt(k) = \begin{cases}
	        1 & \text{if } i \leq k, \\
	        0 & \text{otherwise,}
	    \end{cases}
	\]
	where we exploit the fact that $v_i$ are ordered decreasingly. The above optimization problem can be simplified to
    \be \label{eq:max_k}
    \max_{k} \left\{ \frac{c + \sum_{i=1}^k v_i }{d + k}: k \in [K] \right\}.	
    \ee
    
    Now we need to show that the objective function $g(k) \Let (c + \sum_{i=1}^k v_i)/(d + k)$ becomes non-increasing once it starts decreasing. Indeed, the incremental improvement in the objective value of~\eqref{eq:max_k} at $k$ can be written as
    \begin{align*}
    \Delta_g (k) = g(k+1)-g(k)&=\frac{c+\sum_{i=1}^{k+1} v_i }{d+k+1}-g(k)\\
    &=\frac{(d+k)g(k)+v_{k+1}}{d+k+1}-g(k)\\
    &=\frac{v_{k+1}-g(k)}{d+k+1}.
    \end{align*}
    If $\Delta_g(k)<0$, this implies that $v_{k+1}<g(k)$. We also know that $v_{k+2}\leq v_{k+1}$. So we can show that:
    \begin{align*}
    \Delta_g(k+1) = g(k+2)-g(k+1)&=\frac{v_{k+2}-g(k+1)}{d+k+2}
    \\
    &=\frac{(d+k+1)v_{k+2}-(d+k)g(k)-v_{k+1}}{(d+k+2)(d+k+1)}\\
    &\leq\frac{(d+k+1)v_{k+1}-(d+k)g(k)-v_{k+1}}{(d+k+2)(d+k+1)}\\
    &=\frac{(d+k)(v_{k+1}-g(k))}{(d+k+2)(d+k+1)}<0.
    \end{align*}
    Moreover, the above line of arguments also reveals that if $v_{k+2} = v_{k+1}$ then both $\Delta_g(k)$ and $\Delta_g(k+1)$ have the same sign. Thus, the value $k\opt$ that maximizes~\eqref{eq:max_k} is also the solution of 
    \[
    \max\{k: \Delta_g(k-1) \ge 0\}.
    \]
    Leveraging on the formula of $\alpha_i\opt(k)$, the solution $\alpha\opt$ of the original fractional linear program has the form
    \begin{align*}
        \forall i: \qquad \alpha_i\opt &= 
        \begin{cases}
            1 & \text{if } \ds v_i > \frac{c + \sum_{j: j < i} v_i}{d + |\{j: j < i\}|},\\
            0 & \text{otherwise,}
        \end{cases}  \\
        &= \begin{cases}
	            1 & \text{if } \ds v_i > \frac{c + \sum_{j: v_j > v_i} v_j}{d + |\{j: v_j > v_i \}|},\\
	            0 & \text{otherwise,}
	        \end{cases}
    \end{align*}
    where the second equality comes from the ordering of $v_i$. This observation completes the proof.
	\end{proof}

	\begin{proof}[Proof of Theorem~\ref{thm:infty-refor}]
	    A conditional measure $\mu_0$ of $Y$ given $X \in \mc N_{\gamma}(x_0)$ induced by a probability measure $\QQ$ satisfying $\QQ(X \in \mc N_\gamma(x_0)) > 0$ can be written as
	    \[
	        \QQ(\mc N_\gamma(x_0) \times A) = \mu_{0}(A) \QQ(\mc N_\gamma(x_0) \times \mc Y) \qquad \forall A \subseteq \mc Y~\text{measurable}.
	    \]
	    One can rewrite the worst-case conditional expected loss $f(\beta)$ as
	    \begin{align*}
	        f(\beta)= \left\{
	            \begin{array}{cl}
	                \sup & \ds \int_{\mc Y} \ell(y, \beta) ~\mu_{0}(\mathrm{d}y)  \\
	                \st & \QQ \in \mbb B_\rho^\infty,~\QQ(\mc N_\gamma(x_0) \times \mc Y) > 0 \\[1ex]
	                & \QQ(\mc N_\gamma(x_0) \times A) = \mu_{0}(A) \QQ(\mc N_\gamma(x_0) \times \mc Y) \qquad \forall A \subseteq \mc Y~\text{measurable}.
	            \end{array}
	        \right.
	    \end{align*}
	    By decomposing the measure $\QQ$ using the set of probability measures $\pi_i$ and exploiting the definition of the type-$\infty$ Wasserstein distance as in the proof of Proposition~\ref{prop:infty-vanilla-set}, we have
	    \begin{align*}
	        f(\beta)= \left\{
        		\begin{array}{cll}	
        			\sup & \ds \int_{\mc Y} \ell(y, \beta) ~\mu_{0}(\mathrm{d}y)\\
        			\st & \mu_{0} \in \mc M(\mc Y), \, \pi_i \in \mc M(\mc X \times \mc Y) \quad \forall i \in [N] \\
        			&  \ds \sum_{i \in [N]} \pi_i(\mc N_\gamma(x_0) \times \mc Y) > 0 \\
        			& \ds \sum_{i \in [N]} \pi_i(\mc N_\gamma(x_0) \times A) = \mu_{0}(A) \sum_{i \in [N]} \pi_i(\mc N_\gamma(x_0) \times \mc Y) &\forall A \subseteq \mc Y \text{ measurable}\\
        			& \DD_{\mc X}(x, \wh x_i) + \DD_{\mc Y}(y, \wh y_i) \leq \rho \quad \forall (x, y) \in \mathrm{supp}(\pi_i)& \forall i \in [N].
        		\end{array}
        		\right.
	    \end{align*}
		For any set of feasible solutions $\{ \pi_i\}_{i \in [N]}$, we have $\sum_{i \in [N]} \pi_i(\mc N_\gamma(x_0) \times \mc Y) > 0$. We can thus re-express $\mu_{0}(A)$ for any Borel measurable set $A \subseteq \mc Y$ as
		\[
		\mu_{0}(A) = \frac{\sum_{i \in [N]} \pi_i(\mc N_\gamma(x_0) \times A) }{\sum_{i \in [N]} \pi_i(\mc N_\gamma(x_0) \times \mc Y) }    \quad \forall A \subseteq \mc Y \text{ measurable}.
		\]
		Thus, we can eliminate the variables $\mu_{0}$ from the above optimization problem to obtain the equivalent representation
		\be \label{eq:refor-2}
		f(\beta) = \left\{
		\begin{array}{cll}
			\sup & \displaystyle \frac{1}{\sum_{i \in [N]} \pi_i (\mc N_\gamma(x_0) \times \mc Y)} \sum_{i \in [N]} \int_{\mc Y} \ell(y, \beta)~ \pi_i(\mc N_\gamma(x_0) \times \dd y) \\
			\st  & \pi_i \in \mc M(\mc X \times \mc Y) & \forall i \in [N] \\
			& \DD_{\mc X}(x, \wh x_i) + \DD_{\mc Y}(y, \wh y_i) \leq \rho \quad \forall (x, y) \in \mathrm{supp}(\pi_i) & \forall i \in [N]\\
			& \sum_{i \in [N]} \pi_i(\mc N_\gamma(x_0) \times \mc Y) >0.
		\end{array}
		\right.
		\ee
		We now show that
		problem~\eqref{eq:refor-2} now can be written as
		\be \label{eq:refor-3}
		f(\beta) = \left\{
		\begin{array}{cll}
			\sup & \displaystyle \frac{1}{\sum_{i \in [N]}  \alpha_i} \sum_{i \in [N]}  \alpha_i v_i\opt(\beta) \\
			\st &  \alpha \in [0, 1]^N \\
			& \alpha_i = 1 \text{ if } \DD_{\mc X}(x_0, \wh x_i) + \rho \le \gamma \\
			& \alpha_i = 0 \text{ if } \DD_{\mc X}(x_0, \wh x_i) > \rho + \gamma  \\
			& \sum_{i \in [N]}  \alpha_i > 0,
		\end{array}
		\right.
		\ee
		where the value $v_i\opt(\beta)$ is calculated as
		\[
		v_i\opt(\beta) = \sup\left\{\ell(y_i, \beta)~:y_i\in \mc Y,~\DD_{\mc Y}(y_i, \wh y_i) \leq \rho - \DD_{\mc X}(\wh x_i^p, \wh x_i) \right\}.
		\]
The equivalence between the supremum problems  \eqref{eq:refor-2} and \eqref{eq:refor-3} can be shown in two steps. First, for \eqref{eq:refor-2} $\leq$ \eqref{eq:refor-3}, given any feasible solution of \eqref{eq:refor-2}, one can construct a feasible solution of \eqref{eq:refor-3} using $\alpha_i = \pi_i (\mc N_\gamma(x_0) \times \mc Y)$. For this candidate we have 
\[
\frac{\sum_{i \in [N]} \int_{\mc Y} \ell(y, \beta)~ \pi_i(\mc N_\gamma(x_0) \times \dd y)}{\sum_{i \in [N]} \pi_i (\mc N_\gamma(x_0) \times \mc Y)}  \leq \frac{\sum_{i \in [N]} \alpha_i \ell(y_i\opt, \beta)}{\sum_{i \in [N]} \alpha_i} \,.
\]		
Alternatively, given a feasible solution for \eqref{eq:refor-3}, one can construct the following feasible solution for \eqref{eq:refor-2}: for any $\epsilon > 0$, let $y_i^\epsilon \in \mc Y$ be such that $\DD_{\mc Y}(y_i^\epsilon, \wh y_i) \le \rho - \DD_{\mc X}(x_0, \wh x_i)$ and $\ell (y_i^\epsilon, \beta) \ge v_i\opt(\beta) - \epsilon$, and let
\[
\forall i \in [N] : \quad \pi_i^\epsilon
= \begin{cases}
\delta_{(\wh x_i^p,y_i^\epsilon)}&\text{if } \DD_{\mc X}( x_0, \wh x_i) + \rho \le \gamma,\\
\alpha_i\delta_{(\wh x_i^p,y_i^\epsilon)}+ (1-  \alpha_i)\delta_{(x_i^r,\wh y_i)} & \text{if }  \DD_{\mc X}(x_0, \wh x_i) > \rho + \gamma,\\
\delta_{(\wh x_i, \wh y_i)}&\text{otherwise,}\end{cases}
\]
where $x_i^r$ is any point such that $\DD_{\mc X}(x_i^r, \wh x_i)\leq \rho$ and $x_i^r\notin \mc N_\gamma(x_0)$. Again, this candidate is feasible in \eqref{eq:refor-2} and we have that
\begin{align*}
    f(\beta) &\ge \Sup{\epsilon > 0} \frac{\sum_{i \in [N]} \int_{\mc Y} \ell(y, \beta)~ \pi_i^\epsilon(\mc N_\gamma(x_0) \times \dd y)}{\sum_{i \in [N]} \pi_i^\epsilon (\mc N_\gamma(x_0) \times \mc Y)}  \\
    &\ge \Sup{\epsilon > 0} \frac{\sum_{i \in [N]} \alpha_i (\ell(y_i\opt, \beta) -\epsilon) }{\sum_{i \in [N]} \alpha_i } \\
    &= \frac{\sum_{i \in [N]} \alpha_i \ell(y_i\opt, \beta)  }{\sum_{i \in [N]} \alpha_i } = \frac{\sum_{i \in [N]} \alpha_i v_i\opt(\beta)  }{\sum_{i \in [N]} \alpha_i }.
\end{align*}

 Let $\mc I$ and $\mc I_1$ be the index sets defined as in~\eqref{eq:I-def1}-\eqref{eq:I-def2}, the value $f(\beta)$ is equal to the optimal value of a fractional linear program
		\begin{subequations}
		\begin{align}
		    f(\beta) &= \max ~\left\{\ds \frac{ \sum_{i \in \mc I} v_i\opt(\beta) \alpha_i }{ \sum_{i \in \mc I} \alpha_i} : \alpha \in [0, 1]^N,
		        \; \alpha_i = 1 ~ \forall i \in \mc I_1, \; \sum_{i \in \mc I} \alpha_i > 0 \right\} \label{eq:f-fractional1}\\
		        &= \max ~\left\{\ds \frac{ \sum_{i \in \mc I_1} v_i\opt(\beta) + \sum_{i \in \mc I_2} v_i\opt(\beta) \alpha_i }{ |\mc I_1| + \sum_{i \in \mc I_2} \alpha_i} : \alpha \in [0, 1]^N,\; \alpha_i = 1 ~ \forall i \in \mc I_1,
		        \; |\mc I_1| + \sum_{i \in \mc I_2} \alpha_i > 0 \right\}. \label{eq:f-fractional2}
		\end{align}
		\end{subequations}
		Notice that the objective function and the constraints of~\eqref{eq:f-fractional2} depend only on $\alpha_i$ for $i \in \mc I$. Suppose that $\mc I_1 \not= \emptyset$, Lemma~\ref{lemma:fractional} indicates that the optimal solution $\alpha\opt$ that solves~\eqref{eq:f-fractional2} is
		\be \label{eq:alpha-opt1}
	        \forall i \in \mc I: \quad \alpha_i\opt = \begin{cases}
	            1 & \text{if } i \in \mc I_1, \\
	            1 & \text{if } \ds v_i\opt(\beta) > \frac{\sum_{i \in \mc I_1} v_i\opt(\beta) + \sum_{j: v_j\opt(\beta) > v_i\opt(\beta)} v_j\opt(\beta)}{|\mc I_1| + |\{j: v_j\opt(\beta) > v_i\opt(\beta) \}|},\\
	            0 & \text{otherwise.}
	        \end{cases} 
	    \ee
		Suppose that $\mc I_1 = \emptyset$, then the optimal solution of problem~\eqref{eq:f-fractional2} is
		\[
		    \forall i \in \mc I: \quad \alpha_i\opt = \begin{cases}
		        1 &\text{if } v_i\opt(\beta) \ge \max_{j \in \mc I_2} v_j\opt(\beta), \\
		        0 & \text{otherwise.}
		    \end{cases}
		\]
		Combining the above two cases, we can rewrite the optimal value of $\alpha$ that solves~\eqref{eq:f-fractional2} as in the statement of the theorem. This completes the proof.
	\end{proof}
	
	\begin{proof}[Proof of Corollary~\ref{corol:vi-opt}]
	    Because $\DD_{\mc Y}$ is an absolute distance, we have
	    \[
	        \{ y_i\in \mc Y: |y_i -  \wh y_i| \le \rho - \DD_{\mc X}(\wh x_i^p, \wh x_i) \}\!=\! [\max\{a, \wh y_i - \rho + \DD_{\mc X}(\wh x_i^p, \wh x_i)\}, \min\{b, \wh y_i + \rho - \DD_{\mc X}(\wh x_i^p, \wh x_i)\} ],
	    \]
	    where the equality follows from $\mc Y = [a, b]$. Because both the $\|\cdot\|_2^2$ and the quantile loss functions are convex, the value $v_i\opt(\beta)$ is thus attained at the extreme points of the interval. Calculating the value of $\ell(\cdot, \beta)$ at these two endpoints and taking the maximum between them completes the proof.
	\end{proof}

	Before proving Proposition~\ref{prop:cond-exp}, we need the following two results which asserts the analytical optimal value of maximizing a convex quadratic functions over a norm ball. These results can be found in the literature, the proof is included here for completeness.
	\begin{lemma}[Convex quadratic maximization over a norm ball] \label{lemma:easy}
		For any $\beta \in \R^m$, $\wh y \in \R^m$ and $r \in \R_+$, the following assertions hold.
		\begin{enumerate}[label=(\roman*)]
		\item \label{lemma:easy-1}
		Over a $\|\cdot\|_2$ ball, we have
		\[
		\sup \left\{ \| y - \beta \|_2^2 : ~ \| y - \wh y \|_2^2 \leq r^2  \right\} = (r + \|\wh y - \beta\|_2)^2.
		\]
		\item \label{lemma:easy-2}
		Over a $\|\cdot\|_\infty$ ball, we have
		\[
		\sup \left\{ \| y - \beta \|_2^2 : ~ \| y - \wh y \|_\infty \leq r  \right\} = \sum_{j \in [m]} \max \left\{ (\wh y_j - \beta_j - r)^2, (\wh y_j - \beta_j + r)^2 \right\},
		\]
		where $\beta_j$ and $\wh y_{j}$ denote the $j$-th element of the vector $\beta$ and $\wh y$, respectively.
		\end{enumerate}
	\end{lemma}
	\begin{proof}[Proof of Lemma~\ref{lemma:easy}]
		We first prove Assertion~\ref{lemma:easy-1}. First, the optimal value is upper bounded by $(r+\|\wh y - \beta\|_2)^2$ because
\[\|y-\beta\|_2 \leq \|y-\wh y\|_2 + \|\wh y - \beta\|_2 \leq r+\|\wh y - \beta\|_2\]
by triangle inequality.
Yet, it is equal to that amount since that amount is attained when $y=\wh y + r(\wh y-\beta)/\|\wh y - \beta\|_2$.	
		
		Consider now Assertion~\ref{lemma:easy-2}. Using a change of variables $z \leftarrow y - \beta$ and a change of parameters $w \leftarrow \wh y - \beta$, we find
		\be \label{eq:geometric-refor-2}
		\sup \left\{ \| y - \beta \|_2^2 : ~ \| y - \wh y \|_\infty \leq r \right\} = \max \left\{ \| z \|_2^2: \| z - w \|_\infty \leq r
		\right\},
		\ee
		where the maximization operators are justified by Weierstrass' maximum value theorem~\cite[Theorem~2.43]{ref:aliprantis06hitchhiker} because the feasible set is compact and the objective function is continuous. By extending the norm constraint into the vector form, we have the equivalence
		\[
		\max \left\{ \| z \|_2^2:  w - r \mathbbm{1}_m \leq z \leq w + r \mathbbm{1}_m \right\},
		\]
		where the inequalities in the constraints are understood as element-wise inequalities, and $\mathbbm{1}_m$ is an $m$-dimensional vector of ones. This maximization problem is separable in the decision variables and can be decomposed into $m$ independent univariate subproblems of the form
		\[
		\max\left\{ z_j^2 : w_j - r \leq z_j \leq w_j + r\right\} 
		\]
		for each $j \in [m]$. It is easy to verify that the optimal value of each univariate subproblem is equal to
		\[
		\max \left\{ (w_j - r)^2, (w_j + r)^2 \right\},
		\]
		and summing up the optimal values over $j$ completes the proof.
	\end{proof}

	We are now ready to prove Proposition~\ref{prop:cond-exp}.
	
	\begin{proof}[Proof of Proposition~\ref{prop:cond-exp}]
	    Following from equation~\eqref{eq:f-fractional1} in the proof of Theorem~\ref{thm:infty-refor}, we have
	    \[
	    f(\beta) = \max ~\left\{\ds \frac{ \sum_{i \in \mc I} v_i\opt(\beta) \alpha_i }{ \sum_{i \in \mc I} \alpha_i} : \alpha \in [0, 1]^N,
		        \; \alpha_i = 1 ~ \forall i \in \mc I_1, \; \sum_{i \in \mc I} \alpha_i > 0 \right\}
		 \]

	    By applying the Charnes-Cooper transformation~\cite{ref:charnes1962programming} with
		\[
		    z_i = \frac{\alpha_i}{\sum_{i \in \mc I} \alpha_i }, \quad \text{and} \quad t = \frac{1}{\sum_{i \in \mc I} \alpha_i}
		\]
		to reformulate this fractional linear problem, we have 
		\begin{align*}
			f(\beta) &= \left\{
			    \begin{array}{cll}
			        \max & \sum_{i \in \mc I} v_i\opt(\beta) z_i \\
			        \st & \sum_{i \in \mc I} z_i = 1, ~ t \ge 0 \\
			        &z_i - t = 0 & \forall i \in \mc I_1 \\
			        & 0 \le z_i \le t & \forall i \in \mc I_2. 			    \end{array}
			\right. \\
			&= \left\{
			    \begin{array}{cll}
			        \min & \lambda \\
			        \st & \lambda \in \R,~u_i \in \R~\forall i \in \mc I_1,~u_i \in \R_+~\forall i \in \mc I_2 \\
			        & \lambda + u_i \ge v_i\opt(\beta) \quad \forall i \in \mc I \\
			        & \sum_{i \in \mc I} u_i \le 0,
			    \end{array}
			\right. 
		\end{align*}
		where the second equality follows from linear programming duality. Using the last minimization reformulation of $f(\beta)$, problem~\eqref{eq:local_DRO} is now equivalent to
		\[
		\Min{\beta}~f(\beta) = 
		\left\{
			    \begin{array}{cll}
			        \min & \lambda \\
			        \st & \beta \in \R^m,~\lambda \in \R,~u_i \in \R~\forall i \in \mc I_1,~u_i \in \R_+~\forall i \in \mc I_2 \\
			        & \lambda + u_i \ge v_i\opt(\beta) \quad \forall i \in \mc I \\
			        & \sum_{i \in \mc I} u_i \le 0,
			    \end{array}
			\right. 
		\]
		
		When $\DD_{\mc Y}$ is a 2-norm, each value $v_i\opt (\beta)$ calculated from~\eqref{eq:vi-def} becomes
		\[
		v_i\opt(\beta) = \sup \left\{ \| y - \beta \|_2^2 : ~ \| y - \wh y_i \|_2 \leq \rho - \DD_{\mc X}(\wh x_i^p, \wh x_i)  \right\} \qquad \forall i \in [N].
		\]
		For any $i \in \mc I$, the value $v_i\opt(\beta)$ is finite and $v_i\opt(\beta)$ can be re-expressed by exploiting Lemma~\ref{lemma:easy}\ref{lemma:easy-1} as
		\[
		v_i\opt(\beta) = \left( \rho - \DD_{\mc X}(\wh x_i^p, \wh x_i) + \| \wh y_i - \beta \|_2 \right)^2.
		\]
		Problem~\eqref{eq:local_DRO} is now equivalent to
		\begin{align}
		    \begin{array}{cll}
		        \min & \lambda \\
		        \st & \beta \in \R^m,\;\lambda \in \R,\;u_i \in \R~\forall i \in \mc I_1,\;u_i \in \R_+~\forall i \in \mc I_2 \\
		        & \lambda + u_i \ge \left( \rho - \DD_{\mc X}(\wh x_i^p, \wh x_i) + \| \wh y_i - \beta \|_2 \right)^2 \quad \forall i \in \mc I \\
		        & \sum_{i \in \mc I} u_i \le 0.
		    \end{array}
		\end{align}
		To obtain a second-order cone program formulation, it now suffices to add the hypergraph formulation $t_i \ge \| \wh y_i - \beta \|_2$ with $t_i \ge 0$, and reformulate the quadratic constraint into a second-order cone constraint using results from~\cite[Section~2]{ref:alizadeh2003second}. This completes the proof for claim~\ref{item:cond-exp-2}. 
		
		We now proceed to prove claim~\ref{item:cond-exp-infty}. When $\DD_{\mc Y}$ is the $\infty$-norm, each value $v_i\opt(\beta)$ becomes
		\begin{align*}
		v_i\opt(\beta) = \sup \left\{ \| y - \beta \|_2^2 : ~ \| y - \wh y_i \|_\infty \leq \rho - \DD_{\mc X}(\wh x_i^p, \wh x_i)  \right\} \qquad \forall i \in [N].
		\end{align*}
		For any $i \in \mc I$, the value $v_i\opt(\beta)$ is finite and $v_i\opt(\beta)$ can be re-expressed using  Lemma~\ref{lemma:easy}\ref{lemma:easy-2} as
		\[
		v_i\opt(\beta) = \sum_{j \in [m]} \max \left\{ (\wh y_{ij} - \beta_j - \rho + \DD_{\mc X}(\wh x_i^p, \wh x_i))^2, (\wh y_{ij} - \beta_j + \rho - \DD_{\mc X}(\wh x_i^p, \wh x_i))^2 \right\}.
		\] 
		By adding auxiliary variables $T_{ij}$ with the constraints
		\[
		(\wh y_{ij} - \beta_j - \rho + \DD_{\mc X}(\wh x_i^p, \wh x_i))^2 \leq T_{ij}^2, \quad \text{and} \quad (\wh y_{ij} - \beta_j + \rho - \DD_{\mc X}(\wh x_i^p, \wh x_i))^2 \le T_{ij}^2,
		\]
		problem~\eqref{eq:local_DRO} is now equivalent to
		\[
		   \begin{array}{cll}
		        \min & \lambda \\
		        \st & \beta\in \R^m,\;\lambda \in \R,\; T \in \R_+^{|\mc I| \times m},\;u_i \in \R \;\forall i \in \mc I_1,\;u_i \in \R_+\;\forall i \in \mc I_2 \\
		        & \sum_{i \in \mc I} u_i \le 0\\
		        & \lambda + u_i \ge \sum_{j \in [m]} T_{ij}^2 \quad \forall i \in \mc I \\
		        & (\wh y_{ij} - \beta_j - \rho + \DD_{\mc X}(\wh x_i^p, \wh x_i))^2 \leq T_{ij}^2 &\forall (i, j) \in \mc I \times [m] \\
		       &(\wh y_{ij} - \beta_j + \rho - \DD_{\mc X}(\wh x_i^p, \wh x_i))^2 \le T_{ij}^2  &\forall (i, j) \in \mc I \times [m].
		    \end{array}
		\]
		The last two constraints can be re-expressed as linear constraints of the form
		\[
		    \begin{array}{ll}
		    	 -T_{ij} \le \wh y_{ij} - \beta_j - \rho + \DD_{\mc X}(\wh x_i^p, \wh x_i) \leq T_{ij} &\forall (i, j) \in \mc I \times [m] \\
		       - T_{ij} \le \wh y_{ij} - \beta_j + \rho - \DD_{\mc X}(\wh x_i^p, \wh x_i) \le T_{ij}  &\forall (i, j) \in \mc I \times [m].
		     \end{array}
		\]
		Formulating the quadratic constraint $\lambda + u_i \ge \sum_{j \in [m]} T_{ij}^2$ using~\cite[Section~2]{ref:alizadeh2003second} completes the proof.
	\end{proof}

		\begin{proof}[Proof of Proposition~\ref{prop:gradient}]
	For the purpose of this proof, define the following sets
	\[
	    \mc Y_{i} \Let \left\{y_i \in \mc Y: \DD_{\mc Y}(y_i, \wh y_i) \leq \rho - \DD_{\mc X}(\wh x_i^p, \wh x_i) \right\} \qquad \forall i \in \mc I.
	\]
	Because $\DD_{\mc Y}$ is coercive and continuous, each set $\mc Y_i$ is compact. Because the loss function is continuous, there thus exists $y_i\opt$ satisfying $y_i\opt \in \mc Y_i$ and $\ell(y_i\opt, \beta) = v_i\opt(\beta)$ for any $i \in \mc I$. Following from Equation~\eqref{eq:f-fractional1} in the proof of Theorem~\ref{thm:infty-refor}, we have
	\begin{align*}
	    f(\beta) &= \max ~\left\{\ds \frac{ \sum_{i \in \mc I} v_i\opt(\beta) \alpha_i }{ \sum_{i \in \mc I} \alpha_i} : \alpha \in [0, 1]^N,
		        \; \alpha_i = 1 ~ \forall i \in \mc I_1, \; \sum_{i \in \mc I} \alpha_i > 0 \right\} \\
		        &= \max ~\left\{\ds \frac{ \sum_{i \in \mc I} \ell(y_i, \beta) \alpha_i }{ \sum_{i \in \mc I} \alpha_i} : \alpha \in [0, 1]^N,
		        \; \alpha_i = 1 ~ \forall i \in \mc I_1, \; \sum_{i \in \mc I} \alpha_i > 0,\; y_i \in \mc Y_i~\forall i \in \mc I \right\}.
	\end{align*}
	If $\mc I_1 = \emptyset$, then we have
	\[
	    f(\beta) = \ell(y_{i\opt}, \beta)  \quad \forall i\opt \in \arg \max_{i \in \mc I_2} v_i\opt(\beta),
	\]
	and a subgradient of $f$ is $\partial f(\beta) = \partial_\beta \ell(y_{i\opt}, \beta)$ for any $i\opt \in \arg \max_{i \in \mc I_2} v_i\opt(\beta)$. By incorporating the optimal value of $\alpha$ in the statement of Theorem~\ref{thm:infty-refor}, we have  $\partial f(\beta) = \alpha_i \partial_\beta \ell(y_i\opt, \beta)$.
	
	If $\mc I_1 \not= \emptyset$, then we have
	\begin{align*}
	    f(\beta) &= \max ~\left\{\ds \frac{ \sum_{i \in \mc I_1} v_i\opt(\beta) + \sum_{i \in \mc I_2} v_i\opt(\beta) \alpha_i }{ |\mc I_1| + \sum_{i \in \mc I_2} \alpha_i} : \alpha \in [0, 1]^N,\; \alpha_i = 1 ~ \forall i \in \mc I_1 \right\} \\
	    &= \max ~\left\{\ds \frac{ \sum_{i \in \mc I_1} \ell(y_i, \beta) + \sum_{i \in \mc I_2} \ell(y_i, \beta) \alpha_i }{ |\mc I_1| + \sum_{i \in \mc I_2} \alpha_i} : \alpha \in [0, 1]^N,\; \alpha_i = 1 ~ \forall i \in \mc I_1,\; y_i \in \mc Y_i ~\forall i \in \mc I \right\}
	\end{align*}
	Notice that the function 
    \[
    \beta \mapsto \ds \frac{ \sum_{i \in \mc I_1} \ell(y_i, \beta) + \sum_{i \in \mc I_2} \ell(y_i, \beta) \alpha_i }{ |\mc I_1| + \sum_{i \in \mc I_2} \alpha_i}
    \] 
    is convex for any feasible value of $(\alpha, y)$ in the above optimization problem. Moreover, by Tychonoff's theorem~\cite[Theorem~2.61]{ref:aliprantis06hitchhiker}, the feasible set of the above  optimization problem is a compact set in the product topology. One can now apply \cite[Proposition~A.22]{ref:bertsekas1971control} to conclude that a subgradient of $f$ in this case is
    \[
        \partial f(\beta) = \frac{ \sum_{i \in \mc I_1} \partial_\beta \ell(y_i, \beta) + \sum_{i \in \mc I_2} \partial_\beta \ell(y_i, \beta) \alpha_i }{ |\mc I_1| + \sum_{i \in \mc I_2} \alpha_i}.
    \]
    Combining the two cases, we have the postulated result.
	\end{proof}

	\subsection{Proofs of Section~\ref{sect:guarantee}}

    \begin{proof}[Proof of Proposition~\ref{prop:finite}]
        Under the conditions of the proposition, we have $\PP(X \in \mc N_{\gamma}(x_0)) > 0$ because $\PP$ admits a density, and that $\mc N_{\gamma}(x_0) \cap \mc X$ is a set with non-empty interior for any $\gamma > 0$. The proof now follows trivially from~\cite[Theorem~1.1]{ref:trillos2015rate}. Indeed, under the conditions of the proposition, with probability of at least $1 - O(N^{-c})$, we have $\PP \in \mbb B_\rho^\infty$, and hence the bound follows.
    \end{proof}
    \begin{proof}[Proof of Example~\ref{example:non-consistency}]
    For the purpose of this proof, we let $\PP^{\infty}
        = \PP\otimes\PP\otimes\cdots$ be the joint distribution of 
        $(\wh x_1, \wh y_1),
        (\wh x_2, \wh y_2),\cdots$.
    The selection of parameter $\gamma = 0$ implies that $\mc I$ = $\mc I_2$, and for any fixed $\rho>0$ we have $\PP^{\infty}
    \left(\lim_{N\rightarrow\infty}
    |\mc I| = +\infty
    \right) = 1$ by Borel-Cantelli lemma. In this example, the DRO problem is feasible if $\mc I$ is nonempty, and we have an explicit optimal solution
    \[
    \beta_N\opt = 
    \frac{1}{2}
    \min_{i\in\mc I} 
    \left\{
    \wh y_i - \rho + 
    \DD_{\mc X}(\wh x_i, x_0)
    \right\}
    + 
    \frac{1}{2}
    \max_{i\in\mc I} 
    \left\{
    \wh y_i + \rho - 
    \DD_{\mc X}(\wh x_i, x_0)
    \right\}
    \]
    Notice that with probability $1$ we have
    \[
    \min_{i\in\mc I} 
    \left\{
    \wh y_i - \rho + 
    \DD_{\mc X}(\wh x_i, x_0)
    \right\} \geq -\rho
    \mbox{   and   }
    \max_{i\in\mc I} 
    \left\{
    \wh y_i + \rho - 
    \DD_{\mc X}(\wh x_i, x_0)
    \right\}\geq 
    \max_{i\in\mc I} 
    \left\{
    \wh y_i\right\}.
    \]
    Consequently we have 
    $\beta_N\opt\geq
    \frac{1}{2}\max_{i\in\mc I} 
    \left\{\wh y_i\right\} - \frac{1}{2}\rho.$ For all $y>0$, we have 
    \begin{align*}
    \PP^{\infty}
    \left(\lim_{N\rightarrow\infty}
    \beta_{N}\opt > y
    \right)
    &\geq 
    \PP^{\infty}
    \left(\lim_{N\rightarrow\infty}
    \max_{i\in\mc I}\left\{\wh y_i\right\} > 2y+\rho
    \right)
    = \lim_{N\rightarrow\infty}
    \PP^{\infty}
    \left(
    \max_{i\in\mc I}\left\{\wh y_i\right\} > 2y+\rho
    \right)\\
    & = 
    \lim_{N\rightarrow\infty}
    1-
    \PP(Y\leq2y+\rho)^{|\mc I|}= 1.
    \end{align*} 
    Let $y$ tend to infinity concludes the proof.
    \end{proof}

    Before proving Proposition~\ref{prop:max-var}, we first present the following minimax result.
	\begin{lemma}[Minimax result] \label{lemma:minimax}
	    Suppose that $\ell(y, \cdot)$ is convex and coercive for any $y \in \mc Y$, and that $\DD_{\mc Y}(\cdot, \wh y)$ is convex and coercive for any $\wh y$. For any $\rho \ge \min_{i \in [N]} \kappa_{i, \gamma}$, we have
	    \begin{align*}
	        &\Min{\beta \in \R^m} \sup_{\QQ \in \mbb B_\rho^\infty, \QQ (X \in \mc N_{\gamma}(x_0)) > 0 } \EE_{\QQ} \big[ \ell(Y, \beta) | X \in \mc N_{\gamma}(x_0) \big] \\
	        & \hspace{5cm}=  \sup_{\QQ \in \mbb B_\rho^\infty, \QQ (X \in \mc N_{\gamma}(x_0)) > 0 } \Min{\beta \in \R^m} \EE_{\QQ} \big[ \ell(Y, \beta) | X \in \mc N_{\gamma}(x_0) \big].
	    \end{align*}
	\end{lemma}
	
	To facilitate the proof of Lemma~\ref{lemma:minimax}, we define the following conditional ambiguity set induced by $\mbb B_\rho^\infty$ as
	\be \label{eq:Bx-def}
	\mc B_{x_0, \gamma}(\mbb B_\rho^\infty) \Let \left\{ \mu_{0} \in \mc M(\mc Y): \!\!
	\begin{array}{l}
		\exists \QQ \in \mbb B_\rho^\infty,~\QQ(\mc N_\gamma(x_0) \times \mc Y) > 0\\
		\QQ(\mc N_\gamma(x_0) \times A ) = \mu_{0}(A) \, \QQ(\mc N_\gamma(x_0) \times \mc Y) ~~ \forall A\subseteq \mc Y \text{ measurable} 
	\end{array}
	\right\},
	\ee
	where the last constraint defining the set $\mc B_{x_0, \gamma}(\mbb B_\rho^\infty)$ is from the dis-integration of the joint measure into a marginal distribution and the corresponding conditional distributions~\cite[Theorem~9.2.2]{ref:stroock2011probability}.
	
	The proof of Lemma~\ref{lemma:minimax} relies on the following two results which assert the convexity of the joint ambiguity set $\mbb B_\rho^\infty$ and its induced conditional ambiguity set $\mc B_{x_0, \gamma}(\mbb B_\rho^\infty)$.
	
	\begin{lemma}[Convexity of $\mbb B_\rho^\infty$] \label{lemma:convexity-joint}
	    The ambiguity set $\mbb B_\rho^\infty$ is convex.
	\end{lemma}
	\begin{proof}[Proof of Lemma~\ref{lemma:convexity-joint}]
	    Because the nominal probability measure is an empirical measure, the ambiguity set $\mbb B_\rho^\infty$ can be represented as
	    \[
	    \mbb B_\rho^\infty = \left\{
	        \QQ \in \mc M(\mc X \times \mc Y): 
	        \begin{array}{l}
	            \exists \pi_i \in \mc M(\mc X \times \mc Y) ~\forall i \in [N] \text{ such that :} \\
	            \QQ = N^{-1} \sum_{i \in [N]} \pi_i,~ \sum_{i \in [N]} \pi_i(\mc N_\gamma(x_0) \times \mc Y) > 0 \\
	            \DD_{\mc X}(x, \wh x_i) + \DD_{\mc Y}(y, \wh y_i) \leq \rho \quad \forall (x, y) \in \mathrm{supp}(\pi_i) \quad \forall i \in [N]
	        \end{array}
	    \right\}.
	    \]
	    Pick any arbitrary $\QQ^0$ and $\QQ^1$ from $\mbb B_\rho^\infty$. Associated with $\QQ^j$, $j \in \{0, 1\}$ is a collection of probability measures $\{\pi_i^j\} \in \mc M(\mc X \times \mc Y)^N$ satisfying
	    \[
	        \left\{
	            \begin{array}{l}
	                \QQ^j = N^{-1} \sum_{i \in [N]} \pi_i^j,~ \sum_{i \in [N]} \pi_i^j(\mc N_\gamma(x_0) \times \mc Y) > 0 \\
	                \DD_{\mc X}(x, \wh x_i) + \DD_{\mc Y}(y, \wh y_i) \leq \rho \quad \forall (x, y) \in \mathrm{supp}(\pi_i^j) \quad \forall i \in [N].
	            \end{array}
	        \right.
	    \]
	    Consider any convex combination $\QQ^\lambda = \lambda \QQ^1 + (1-\lambda) \QQ^0$ for $\lambda \in (0, 1)$. It is easy to verify that the joint measure $\pi_i^\lambda = \lambda \pi_i^1 + (1- \lambda) \pi_i^0$ for any $i \in [N]$ satisfies
	    \[
	        \left\{
	            \begin{array}{l}
	                \QQ^\lambda = N^{-1} \sum_{i \in [N]} \pi_i^\lambda,~ \sum_{i \in [N]} \pi_i^\lambda(\mc N_\gamma(x_0) \times \mc Y) > 0 \\
	                \DD_{\mc X}(x, \wh x_i) + \DD_{\mc Y}(y, \wh y_i) \leq \rho \quad \forall (x, y) \in \mathrm{supp}(\pi_i^\lambda) \quad \forall i \in [N],
	            \end{array}
	        \right.
	    \]
	    where the last constraint is satisfied by noticing that $\mathrm{supp}(\pi_i^\lambda) = \mathrm{supp}(\pi_i^0) \cup \mathrm{supp}(\pi_i^1)$.
	    This observation implies that $\QQ^\lambda \in \mbb B_\rho^\infty$. 
	\end{proof}

	\begin{lemma}[Convexity of $\mc B_{x_0, \gamma}(\mbb B_\rho^\infty)$] \label{lemma:convexity}
	     The conditional ambiguity set $\mc B_{x_0, \gamma}(\mbb B^\infty_\rho)$ is convex.
	\end{lemma}
	\begin{proof}[Proof of Lemma~\ref{lemma:convexity}]
	    Let $\mu_0^0, \mu_0^1 \in \mc B_{x_0, \gamma}(\mbb B_\rho^\infty)$ be two arbitrary probability measures. Associated with each $\mu_0^j$, $j \in \{0, 1\}$, is a corresponding joint measure $\QQ^j \in \mc M(\mc X \times \mc Y)$ such that
	    \[
	        \QQ^j(\mc N_\gamma(x_0) \times \mc Y) > 0 \quad \text{and} \quad 
		        \ds\frac{\QQ^j(\mc N_\gamma(x_0) \times A)}{\QQ^j(\mc N_\gamma(x_0) \times \mc Y)} = \mu_0^j(A) \quad \forall A \subseteq \mc Y~\text{measurable}.
	    \]
	    Select any $\lambda \in (0, 1)$. We proceed to show that $\mu_0^\lambda = \lambda \mu_0^1 + (1-\lambda) \mu_0^0 \in \mc B_{x_0, \gamma}(\mbb B_\rho^\infty)$. Indeed, consider the joint measure 
	    \[
	        \QQ^\lambda = \theta \QQ^1 + (1-\theta) \QQ^0
	    \]
	    with $\theta$ being defined as
	    \[
	        \theta = \frac{\lambda \QQ^0(\mc N_\gamma(x_0) \times \mc Y)}{\lambda \QQ^0(\mc N_\gamma(x_0) \times \mc Y) + (1-\lambda) \QQ^1(\mc N_\gamma(x_0) \times \mc Y)} \in [0, 1].
	    \]
	    By definition, we have $\QQ^\lambda(\mc N_\gamma(x_0) \times \mc Y) > 0$, and by convexity of $\mbb B_\rho^\infty$ from Lemma~\ref{lemma:convexity-joint}, we have $\QQ^\lambda \in \mbb B_\rho^\infty$. Moreover, we have for any set $A \subseteq \mc Y$ measurable,
	    \begin{align*}
	        \frac{\QQ^\lambda(\mc N_\gamma(x_0) \times A)}{\QQ^\lambda(\mc N_\gamma(x_0) \times \mc Y)} &= \frac{\theta \QQ^1(\mc N_\gamma(x_0) \times A) + (1-\theta) \QQ^0(\mc N_\gamma(x_0) \times A)}{\theta \QQ^1(\mc N_\gamma(x_0) \times \mc Y) + (1-\theta) \QQ^0(\mc N_\gamma(x_0) \times \mc Y)} \\
	        &= \frac{\lambda \QQ^0(\mc N_\gamma(x_0) \times \mc Y) \QQ^1(\mc N_\gamma(x_0) \times A) + (1-\lambda) \QQ^1(\mc N_\gamma(x_0) \times \mc Y) \QQ^0(\mc N_\gamma(x_0) \times A)}{\QQ^0(\mc N_\gamma(x_0) \times \mc Y) \QQ^1(\mc N_\gamma(x_0) \times \mc Y)} \\
	        &= \frac{\lambda \QQ^1(\mc N_\gamma(x_0) \times A)}{\QQ^1(\mc N_\gamma(x_0) \times \mc Y)} + \frac{(1-\lambda) \QQ^0(\mc N_\gamma(x_0) \times A)}{\QQ^0(\mc N_\gamma(x_0) \times \mc Y)} \\
	        &= \lambda \mu_0^1(A) + (1-\lambda) \mu_0^0(A),
	    \end{align*}
	   where the second equality follows from the definition of $\theta$. This implies that $\mu_0^\lambda \in \mc B_{x_0, \gamma}(\mbb B_\rho^\infty)$, and further implies the convexity of $\mc B_{x_0, \gamma}(\mbb B_\rho^\infty)$. 
	\end{proof}
	
	We are now ready to prove Lemma~\ref{lemma:minimax}.
	
	\begin{proof}[Proof of Lemma~\ref{lemma:minimax}]
	    By the definition of the conditional ambiguity set $\mc B_{x_0, \gamma}(\mbb B_\rho^\infty)$, it suffices to prove the equivalence
	    \[
	     \Min{\beta \in \R^m}  \Sup{\mu_0 \in \mc B_{x_0, \gamma}(\mbb B_\rho^\infty)}~\EE_{\mu_0}[\ell(Y, \beta)] =   \Sup{\mu_0 \in \mc B_{x_0, \gamma}(\mbb B_\rho^\infty)} \Min{\beta \in \R^m}~\EE_{\mu_0}[\ell(Y, \beta)].
	    \]
	
	    First, consider the mapping $\beta \mapsto \sup_{\mu_0 \in \mc B_{x_0, \gamma}(\mbb B_\rho^\infty)}~\EE_{\mu_0}[\ell(Y, \beta)]$. The properties of $\ell$ implies that this mapping is lower semi-continuous and coercive. As a consequence, without loss of optimality, we can restrict the feasible set $\beta$ to some convex, compact ball $\mc S \Let \{ \beta: \| \beta\|_2 \le R\}$ for some radius $R \in \R_{++}$ sufficiently big.

	    We now consider the mapping $\mu_0 \mapsto \EE_{\mu_0}[\ell(Y, \beta)]$ parametrized by $\beta$. For any $\beta$, it is a linear function of $\mu_0$, and hence it is concave. It is also weakly continuous. To see this, notice that when $\DD(\cdot, \wh y)$ is coercive, the set
	    \[
	        \mc A \Let \bigcup_{i \in [N]} \{ y: \DD_{\mc Y}(y, \wh y_i) \leq \rho \},
	    \]
	    being a finite union of bounded sets, is bounded. Pick any $\QQ \in \mbb B_\rho^\infty$, by the definition of the type-$\infty$ Wasserstein distance, we have $\QQ(\mc A) = 1$. Consider the conditional measure $\mu_0^{\QQ}$ induced by $\QQ$, then we have
	    \[
	        \mu_0^{\QQ}(\mc A \cap \mc Y) = \frac{\QQ(\mc N_\gamma(x_0) \times (\mc A \cap \mc Y))}{\QQ(\mc N_\gamma(x_0) \times \mc Y)} \ge \frac{\QQ(\mc N_\gamma(x_0) \times (\mc A \cap \mc Y))}{\QQ(\mc N_\gamma(x_0) \times (\mc A \cap \mc Y))}= 1,
	    \]
	    which implies that $\mu_0^{\QQ}$ has a bounded support. This implies that $\mc B_{x_0, \gamma}(\mbb B_\rho^\infty) \subseteq \mc M(\mc A)$, where $\mc M(\mc A)$ is the set of all probability measures supported on a bounded set $\mc A$.
	    Because $\ell(\cdot, \beta)$ is continuous, there exists a bound $U\in \R_{++}$ such that $|\ell(y, \beta)|\le U$ for every $y \in \mc A$. Define now the function $\ell_U(\cdot, \beta) = \max\{-U,\min\{\ell(\cdot, \beta), U\}\}$, which is continuous and bounded. Consider any sequence of conditional measures $\{\mu_0^k\} \in \mc M(\mc A)$ that weakly converges to $\mu_0^\infty$, we have
	    \begin{align*}
	        \lim_{k \uparrow \infty} \EE_{\mu_0^k}[\ell(Y, \beta)] = \lim_{k \uparrow \infty} \EE_{\mu_0^k}[\ell_U(Y, \beta)] =  \EE_{\mu_0^\infty}[\ell_U(Y, \beta)] = \EE_{\mu_0^\infty}[\ell(Y, \beta)],
	    \end{align*}
	    which implies that the function $\mu_0 \mapsto \EE_{\mu_0}[\ell(Y, \beta)]$ is weakly continuous over $\mc M(\mc A)$.
	    
	    This line of argument suggests that 
	    \begin{subequations}
	    \begin{align}
	        \Min{\beta} \Sup{\mu_0 \in \mc B_{x_0, \gamma}(\mbb B_\rho^\infty)}~\EE_{\mu_0}[\ell(Y, \beta)] &= \Min{\beta: \| \beta\|_2\le R} \Sup{\mu_0 \in \mc B_{x_0, \gamma}(\mbb B_\rho^\infty)}~\EE_{\mu_0}[\ell(Y, \beta)] \notag\\
	        &=\Sup{\mu_0 \in \mc B_{x_0, \gamma}(\mbb B_\rho^\infty)} \Min{\beta: \| \beta\|_2 \le R} ~\EE_{\mu_0}[\ell(Y, \beta)] \label{eq:minimax-2}\\
	        &= \Sup{\mu_0 \in \mc B_{x_0, \gamma}(\mbb B_\rho^\infty)} \Min{\beta} ~\EE_{\mu_0}[\ell(Y, \beta)], \label{eq:minimax-3}
	    \end{align}
	    \end{subequations}
	    where equality~\eqref{eq:minimax-3} follows from the coercivity of the loss function, thus the constraint on $\beta$ can be dropped for $R$ sufficiently big. Equality~\eqref{eq:minimax-2} holds by Sion's minimax theorem~\cite{ref:sion1958minimax}. This finishes the proof.
	\end{proof}

	\begin{proof}[Proof of Proposition~\ref{prop:max-var}]
        Because the loss function is coercive and convex in $\beta$, we have
        \begin{align*}
            \Min{\beta \in \R} \Sup{\mu_0 \in \mc B_{x_0, \gamma}(\mbb B_\rho^\infty)}~\EE_{\mu_0}[(Y - \beta)^2] &= \Sup{\mu_0 \in \mc B_{x_0, \gamma}(\mbb B_\rho^\infty)} \Min{\beta \in \R}~\EE_{\mu_0}[(Y - \beta)^2] \\ &= \Sup{\mu_0 \in \mc B_{x_0, \gamma}(\mbb B_\rho^\infty)}~\EE_{\mu_0}[(Y - \EE_{\mu_0}[Y])^2] \\
            &= \textrm{Variance}_{\mu_0\opt}(Y),
        \end{align*}
        where the first equality follows from Lemma~\ref{lemma:minimax}, the second equality follows from the fact that for any $\mu_0 \in \mc B_{x_0, \gamma}(\mbb B_\rho^\infty)$, the estimate $\beta\opt(\mu_0) = \EE_{\mu_0}[Y]$ minimizes the objective $\EE_{\mu_0}[(Y- \beta)^2]$. The last equality follows from the definition of $\mu_0\opt$. 
        
        Let $\beta\opt$ be the optimal estimate that solves~\eqref{eq:local_DRO}, we now have
    \begin{align*}
        \textrm{Variance}_{\mu_0\opt}(Y)
    &=\Sup{\mu_0 \in \mc B_{x_0, \gamma}(\mbb B_\rho^\infty)}~\EE_{\mu_0}[(Y - \beta\opt)^2] \\
    &\geq
     \EE_{\mu_0\opt}[(Y - \beta\opt)^2] 
    =
    \textrm{Variance}_{\mu_0\opt}(Y)
    +  (\beta\opt - \EE_{\mu_0\opt}[Y])^2,
    \end{align*}
    where the last equality follows from the bias-variance decomposition. This implies that $\beta\opt = \EE_{\mu_0\opt}[Y]$ and completes the proof.
    \end{proof}
    
    
	\section{Golden-section Search for Univariate Conditional Estimate}
	
	We elaborate here on the procedure of applying a golden-section search to solve a one-dimensional local conditional estimation with a convex loss function $\ell$. We suppose that $\mc Y = [a, b]$ for some finite values $-\infty < a < b < \infty$, that $\ell(y, \cdot)$ is convex for every $y$ and that we have access to an oracle that solves~\eqref{eq:vi-def}. Given any $\beta$, the worst-case conditional expected loss $f(\beta)$ can be computed using Theorem~\ref{thm:infty-refor}. Algorithm~\ref{alg:golden-section} can be used to find the optimal conditional estimate $\beta\opt$ to any arbitrary precision.
	
	\begin{algorithm}[h]
	\caption{Golden-section Search Algorithm}
	\label{alg:golden-section}
	\begin{algorithmic}
		\STATE {\bfseries Input:} Range $[a, b] \in \R$, tolerance $\epsilon\in \R_{++}$
		\STATE {\bfseries Initialization:} Set $r \leftarrow 0.618$, $\beta_1 \leftarrow a$, $\beta_4 \leftarrow b$
		\WHILE{ $|\beta_4 - \beta_1| > \epsilon$}
		\STATE Set $\beta_2 \leftarrow r \beta_1 + (1-r) \beta_4$, $\beta_3 \leftarrow (1-r) \beta_1 + r \beta_4$
		
		\STATE \textbf{if} $f(\beta_2) \le f(\beta_3)$ \textbf{then} Set $\beta_4 \leftarrow \beta_3$ \textbf{else} Set $\beta_1 \leftarrow \beta_2$ \textbf{endif}
		
		\ENDWHILE
		\STATE Set $\beta\opt \leftarrow (\beta_1 + \beta_4)/2$
		\STATE{\bfseries Output:} $\beta\opt$
	\end{algorithmic}
\end{algorithm}

	\bibliographystyle{abbrv}
    \bibliography{arxiv}


\end{document}